\documentclass[12pt]{article}
\usepackage{amsmath}
\usepackage{times}
\usepackage{graphicx}
\usepackage{color}
\usepackage{multirow}
\usepackage[authoryear]{natbib}
\usepackage{rotating}
\usepackage{bbm}
\usepackage{latexsym}
\usepackage{subfigure}
\graphicspath{{Figures/}}

\usepackage{soul}
\usepackage{url}
\usepackage[hidelinks]{hyperref}
\usepackage[utf8]{inputenc}
\usepackage[small]{caption}
\usepackage{booktabs}
\usepackage{amsfonts} 
\usepackage{nicefrac} 
\usepackage{microtype}

\newcommand{\captionfonts}{\normalsize}

\makeatletter  
\long\def\@makecaption#1#2{%
  \vskip\abovecaptionskip
  \sbox\@tempboxa{{\captionfonts #1: #2}}%
  \ifdim \wd\@tempboxa >\hsize
    {\captionfonts #1: #2\par}
  \else
    \hbox to\hsize{\hfil\box\@tempboxa\hfil}%
  \fi
  \vskip\belowcaptionskip}
\makeatother   

\usepackage{hyperref}

\usepackage{url}
\usepackage{subfigure}

\usepackage{mathtools}
\usepackage{lipsum}
\usepackage{amsfonts}
\usepackage{amsmath}
\usepackage{amssymb}
\usepackage{amsthm}
\usepackage{mathrsfs}
\usepackage[noend]{algpseudocode}

\usepackage{algorithmicx,algorithm}
\usepackage{makecell}
\usepackage{booktabs}
\usepackage{float}
\usepackage{color}

\makeatletter
\def\@endtheorem{\endtrivlist}
\makeatother
\newtheorem{theorem}{Theorem}
\newtheorem{corollary}{Corollary}
\newtheorem{definition}{Definition}
\newtheorem{proposition}{Proposition}
\newtheorem{lemma}{Lemma}

\begin{document}
\hspace{13.9cm}1
\ \vspace{20mm}\\
{\LARGE Learning with Proper Partial Labels}

\ \\
{\bf \large Zhenguo Wu$^{\displaystyle 1}$, \bf \large Jiaqi Lv$^{\displaystyle 2}$, \bf \large Masashi Sugiyama$^{\displaystyle 2, \displaystyle 1}$}\\
{$^{\displaystyle 1}$The University of Tokyo}\\
{$^{\displaystyle 2}$RIKEN AIP}\\
%

{\bf Keywords:} Weakly-supervised learning, partial-label learning, empirical risk minimization

\thispagestyle{empty}
\markboth{}{NC instructions}
\ \vspace{-0mm}\\
\begin{center} {\bf Abstract} \end{center}

 \emph{Partial-label learning} is a kind of weakly-supervised learning with inexact labels, where for each training example, we are given a set of candidate labels instead of only one true label. Recently, various approaches on partial-label learning have been proposed under different generation models of candidate label sets. However, these methods require relatively strong distributional assumptions on the generation models. When the assumptions do not hold, the performance of the methods is not guaranteed theoretically. In this paper, we propose the notion of \emph{properness} on partial labels. We show that this proper partial-label learning framework requires a weaker distributional assumption and includes many previous partial-label learning settings as special cases. We then derive a unified unbiased estimator of the classification risk. We prove that our estimator is risk-consistent, and we also establish an estimation error bound. Finally, we validate the effectiveness of our algorithm through experiments.

\section{Introduction}
\emph{Partial-label learning} (PL) \citep{cour2011learning, zeng2013learning} is a typical weakly-supervised learning framework \citep{WSLbook, zhou2018brief}, where each training instance is associated with a set of candidate labels and one of them is the true label.
The PL problem naturally arises in various real-world scenarios, such as web mining \citep{luo2010learning}, face recognition \citep{zeng2013learning}, birdsong classification \citep{liu2012conditional}, and multimedia content analysis \citep{cour2011learning}.

PL has made tremendous strides by developing along two lines.
The first strand is the tailored methodology and training strategy to disambiguate the candidate labels from a practical standpoint.
They can generally be divided into the \emph{average-based strategy} that treats each candidate label equally during training \citep{cour2011learning,zhang2015solving} and the \emph{identification-based strategy} that purifies each candidate label set on the fly to select the most likely true label in the model training phase \citep{liu2012conditional,zeng2013learning,lv2020progressive,feng2020provably,wen2021leverage}.

The second is engaging in theoretically grounded means of combating the adverse influence incurred by the inexact label annotation.
\citet{liu2014learnability} proposed the \emph{small ambiguity degree condition} for guaranteeing the empirical risk minimization (ERM) \emph{learnability} of the PL problem.
Then, some researchers focused on the \emph{statistical consistency} \citep{mohri2018foundations}, i.e., making the empirical risk computed on partial labels consistent to that computed on ordinary labels, for explaining the cause of success or failure of PL methods from the theoretical perspective.
However, these theoretical studies were based on fairly strict assumptions on the generation procedure of candidate labels.
For example, \citet{feng2020provably} proposed a uniform sampling process that a candidate label set is independently and uniformly sampled given a specific true label.
\citet{wen2021leverage} generalized the uniform assumption to a class-dependent setting.
Furthermore, there is another weakly-supervised learning problem related to PL named complementary-label learning \citep{ishida2017learning}, wherein an instance is equipped with a complementary label. 
A complementary label specifies a class that the instance
does not belong to, so it can be seen as an extreme case of PL. 
Most existing consistent complementary-label learning works \citep{ishida2017learning,ishida2019complementary,feng2020learning} are also based on the uniform generation process of complementary labels.
As a result, these assumptions can hold only in limited situations in the real world and may not be very practical. 
Motivated by the above observations, in this paper, we study the PL problem under a weak distributional assumption on the generation process. 
Specifically, we have the following contributions:
\begin{itemize}
\item We propose the notion of \emph{properness} for PL problems. We show that our \emph{Proper Partial-Label Learning} (PPL) framework requires a weaker distributional assumptions and includes much recent work as special cases, such as \emph{Learning with Complementary Labels} (CL) \citep{ishida2019complementary}, \emph{Learning with Multiple Complementary Labels} (MCL) \citep{feng2020learning}, and \emph{Provably Consistent Partial-Label Learning} (PCPL) \citep{feng2020provably}.
\item We derive a unified risk-consistent method for PPL problems, which is both model-independent and optimizer-independent. Theoretically, we establish an estimation error bound for our method. Experimentally, we demonstrate the effectiveness of the proposed method on various benchmark datasets.
\end{itemize}

\section{Preliminaries}
In this section, we first introduce the formulations of ordinary multi-class classification and PL. Then, we review related work on PL problems.

\subsection{Ordinary Multi-class Classification}
Let $\mathcal{X}\in\mathbb{R}^d$ be the instance space and $\mathcal{Y}:=\{1, 2, \ldots, K\}$ be the label space where $d$ is the dimension of the instance space and $K>2$ is the number of classes. 
We denote by $\boldsymbol{x}\in\mathcal{X}$ and $y\in\mathcal{Y}$ the input and output random variables. 
We assume that the input-output pair $(\boldsymbol{x}, y)\in\mathcal{X}\times\mathcal{Y}$ is sampled independently from an unknown joint probability distribution with density $p(\boldsymbol{x},y)$. 
Let $f:\mathcal{X}\rightarrow\mathbb{R}^K$ be a multi-class classifier, with which we determine a class $\hat{y}$ for input $\boldsymbol{x}$ by $\hat{y} = \arg\max_{y\in \mathcal{Y}}f_y(\boldsymbol{x})$, where $f_y$ is the $y$-th element of $f$. 
The goal of multi-class classification is to learn a multi-class classifier that minimizes the following classification risk:
\begin{equation}\label{eq:risk}
R(f;{\mathcal{L}})=\mathbb{E}_{p(\boldsymbol{x}, y)}\left[{\mathcal{L}}(f(\boldsymbol{x}), y)\right],
\end{equation}
where $\mathbb{E}_{p(\boldsymbol{x}, y)}\left[\cdot\right]$ denotes the expectation over the joint probability density $p(\boldsymbol{x}, y)$ and $\mathcal{L}$ is the loss function. 
Since $p(\boldsymbol{x}, y)$ is often unknown, it is a common practice to conduct ERM.
Assume that we have access to a training dataset $\{(\boldsymbol{x}_i, y_i)\}_{i=1}^n$ where each example is independently drawn from $p(\boldsymbol{x}, y)$, and then we minimize the empirical risk instead: 
\begin{equation*}
	\hat{R}_n(f;{\mathcal{L}})=\frac{1}{n}\sum_{i=1}^n{\mathcal{L}}(f(\boldsymbol{x}_i),y_i).
\end{equation*}

\subsection{Partial-Label Learning}
In PL, we do not have direct access to the true labels $y$. Instead, $\boldsymbol{x}_i$ is only equipped with a set of candidate labels, namely, a partial label. 
Suppose a partially labeled example is denoted by $(\boldsymbol{x}, s)$ where $s_i$ is the candidate label set of $\boldsymbol{x}$ and $s\in\mathcal{S}$.
$\mathcal{S} := \{2^{\mathcal{Y}}\setminus\emptyset\setminus\mathcal{Y}\}$ is the space of partial labels where $2^\mathcal{Y}$ denotes the power set of $\mathcal{Y}$.  
Assume $(\boldsymbol{x}, s)$ is drawn from an unknown probability distribution with density $p(\boldsymbol{x}, s)$. 
Then the key assumption of PL is that the partial label $s$ always includes the true label $y$ \citep{cour2011learning, liu2012conditional, liu2014learnability}:
\begin{equation}\label{pplproperty}
P(y\in s|\boldsymbol{x}, s) = 1,
\end{equation}
where $P$ denotes the probability.
We denote the PL training dataset by $\mathcal{D_\mathrm{PL}} = \{(\boldsymbol{x}_i, s_i)\}_{i=1}^n$, where each example is assumed to be independently drawn from the joint probability density $p(\boldsymbol{x}, s)$. The empirical risk $\hat{R}_n(f;{\mathcal{L}})$ is not accessible from $\mathcal{D_\mathrm{PL}}$, which is the common difficulty for applying ERM in weakly supervised classification. One popular solution is to rewrite the classification risk in the form whose direct unbiased estimator can be computed using weak labels \citep{WSLbook, bao2018classification, charoenphakdee2019symmetric, feng2021pointwise, cao2021learning}. Note that the risk rewrite technique is model-independent and optimizer-independent. This key property enables the algorithms to be highly flexible and scalable to large-scale datasets.

\subsection{Related Work}
Here, we review some previous works related to our proposed method: \emph{Learning with Complementary Labels} (CL) \citep{ishida2019complementary}, \emph{Learning with Multiple Complementary Labels} (MCL) \citep{feng2020learning}, and \emph{Provably Consistent Partial-Label Learning} (PCPL) \citep{feng2020provably}.
\paragraph{Learning with Complementary Labels.}
    A complementary label \citep{ishida2017learning, ishida2019complementary, yu2018learning} is a kind of weak labels indicating an incorrect class of an instance. We let $\bar{y}$ denote the complementary label. By taking $s = \mathcal{Y}\setminus\{\bar{y}\}$, the complementary label can be equivalently expressed as a partial label. CL is an extreme case of PL since the size of the candidate labels is fixed to its maximum, $K-1$: $|s|=K-1$. \cite{ishida2017learning, ishida2019complementary} assumed that the training dataset $\mathcal{D}_\text{CL} = \{(\boldsymbol{x}_i, \bar{y}_i)\}_{i=1}^n$ is independently sampled from the joint probability density \[\bar{p}(\boldsymbol{x}, \bar{y})=\frac{1}{K-1}\sum_{y\neq\bar{y}}p(\boldsymbol{x}, y).\] 
    This assumption implies that \[
        \bar{p}(\bar{y}|\boldsymbol{x}, y) = \frac{1}{K-1}\boldsymbol{1}\{\bar{y}\neq y\},\]
        where $\boldsymbol{1}\{\cdot\}$ is the indicator function. Equivalently, this can be expressed as \[p(s|\boldsymbol{x}, y) = \frac{1}{K-1}\boldsymbol{1}\{y\in s\},\]
    which means that all labels except the correct label $y$ are chosen uniformly to be the complementary label.
    Based on this assumption, the authors derived the following unbiased risk estimator:
    \begin{equation*}
        \hat{R}_\text{CL}(f;{\mathcal{L}}) = \frac{1}{n}\sum_{i=1}^n\left[\sum_{k\neq\bar{y_i}}{\mathcal{L}}(f(\boldsymbol{x}_i), k) - (K-2){\mathcal{L}}(f(\boldsymbol{x}_i),\bar{y}_i)\right].
    \end{equation*}
    
\paragraph{Learning with Multiple Complementary Labels.}
\cite{feng2020learning} generalized the problem to deal with multiple complementary labels. We let $\bar{s}$ denote the set of complementary labels. Similarly, by taking $s = {\mathcal{Y}}\setminus\bar{s}$, the multiple complementary label can be equivalently expressed as a partial label. MCL assumes that we have access to a dataset ${\mathcal{D}}_\text{MCL}=\{(\boldsymbol{x}_i, \bar{s}_i)\}_{i=1}^n$, where each example is independently drawn from the joint probability density $\bar{p}(\boldsymbol{x}, \bar{s})$ given as
\begin{equation*}
    \bar{p}(\boldsymbol{x},\bar{s}) = \sum_{i=1}^{K-1}\bar{Q}_i\bar{p}(\boldsymbol{x}, \bar{s} \mid \left\lvert\bar{s}\right\rvert=i),
\end{equation*}
where $|\bar{s}|$ denotes the size of $\bar{s}$, $\bar{Q}_{i}:= P(|\bar{s}| = i)$, and 
\begin{equation*}
    \bar{p}(\boldsymbol{x}, \bar{s} \mid \left\lvert\bar{s}\right\rvert=i):=
    \begin{cases}
    \frac{1}{\binom{K-1}{i}}\sum_{y\notin \bar{s}}p(\boldsymbol{x}, y) &\text{if }|\bar{s}|=i,\\
    0 &\text{otherwise.}
    \end{cases}
\end{equation*}
This assumption implies that
\[
    p(\bar{s}|\boldsymbol{x}, y) = \frac{\bar{Q}_{|\bar{s}|}}{\binom{K-1}{|\bar{s}|}}\boldsymbol{1}\{y\notin\bar{s}\},
\] and equivalently, that
\[    p(s|\boldsymbol{x}, y) = \frac{Q_{|s|}}{\binom{K-1}{|s| - 1}}\boldsymbol{1}\{y\in s\},
\] where $Q_i = P(|s|=i)$.
Based on this assumption, the authors derived the following unbiased risk estimator:
\begin{equation*}
    \hat{R}_\text{MCL}(f;{\mathcal{L}}) = \frac{1}{n}
    \sum_{i=1}^n\left(\sum_{y\notin\bar{s}_i}{\mathcal{L}}(f(\boldsymbol{x}_i),y)-\frac{K-1-|\bar{s}_i|}{|\bar{s}_i|}\sum_{y\in\bar{s}_i}{\mathcal{L}}(f(\boldsymbol{x}_i), y)\right).
\end{equation*}

\paragraph{Provably Consistent Partial-Label Learning.}
In PCPL~\citep{feng2020provably}, the authors assumed that
\begin{equation}\label{eq:uniform}
    p(s|\boldsymbol{x}, y) = \frac{1}{2^{K-1}-1}\boldsymbol{1}\{y\in s\}.
\end{equation}
This assumption implies that given the true label $y$, the remaining part of the partial label $s$ is uniformly sampled from $\mathcal{S}$.
With this assumption, the authors derived the following unbiased risk estimator:
\begin{equation}\label{pl}
    \hat{R}_\mathrm{PCPL}(f;{\mathcal{L}})=\frac{1}{2n}\sum_{i=1}^n\sum_{j=1}^K\frac{p(y=j|\boldsymbol{x}_i)}{\sum_{k\in s}p(y=k|\boldsymbol{x}_i)}{\mathcal{L}}(f(\boldsymbol{x}_i), j).
\end{equation}
The posterior probability $p(y=k|\boldsymbol{x})$ is not accessible. Therefore, the softmax function was utilized to approximate $p(y=k|\boldsymbol{x})$:
\[
p(y=k|\boldsymbol{x})\approx\frac{e^{f_k(\boldsymbol{x})}}{\sum_{i=1}^Ke^{f_i(\boldsymbol{x})}}.
\]
\section{Method}
\label{sec:method}
\label{chap_Proposed}
   In this section, we first propose the notion of \emph{properness} for PL. Then, we propose a general PL framework called the \emph{proper partial-label learning} (PPL) framework. Within this framework, we derive a unified risk-consistent algorithm for PL problems. Finally, we establish an estimation error bound for the proposed method.

\subsection{Proper Partial-Label Learning} 
Here, we introduce the concept of properness in PL:
\begin{definition}
We say that a PL problem is \emph{proper} if there exists a function $C: {\mathcal{X}}\times{\mathcal{S}}\rightarrow\mathbb{R}$ such that the following condition holds:
\begin{equation}\label{propercond}
p(s|\boldsymbol{x},y) = C(\boldsymbol{x}, s)\boldsymbol{1}\{y\in s\}.
\end{equation}
\end{definition}
The direct interpretation for this definition is two-fold. 
First, an example $(\boldsymbol{x}, y)$ can not generate a partial label $s$ which does not include $y$. This coincides with the basic property~\eqref{pplproperty} of PL. 
Second, when $(\boldsymbol{x}, y)$ is given, the probability of generating $s$ does not depend on $y$ as long as $y\in s$. 
    
The properness assumption is general in the sense that it includes many previous problem settings as special cases. In the \emph{Logistic Stick-Breaking Conditional Multinomial Model} (CMM)~\citep{liu2012conditional}, the authors treated the partial label $s$ as label noise and assumed that $s$ and $\boldsymbol{x}$ are conditionally independent when the true label $y$ is given:
\begin{equation}\label{cmm1}
p(s|\boldsymbol{x},y)=p(s|y).
\end{equation}
This conditional independence assumption is common in modeling label noise~\citep{han2018masking}. Further, the authors proposed the following assumption to model the noise distribution:
\begin{equation}\label{cmm2}
p(s|y) = C(s)\boldsymbol{1}\{y\in s\}.
\end{equation}
Here, the function $C$ only depends on $s$. In the following proposition, we obtain an equivalent expression for~\eqref{cmm1} and~\eqref{cmm2}:

\begin{proposition}\label{propCondInd}
\eqref{cmm1} and~\eqref{cmm2} hold if and only if there exists a function $C:{\mathcal{S}}\rightarrow\mathbb{R}$ such that the following condition holds:
\begin{equation}\label{cmmAssumption}
p(s|\boldsymbol{x}, y) = C(s)\boldsymbol{1}\{y\in s\}.
\end{equation}
\end{proposition}
The proofs are given in Appendix~\ref{Appendix:Proofs}. 
By the above proposition, we can see that the CMM assumptions~\eqref{cmm1} and ~\eqref{cmm2} can be written in the form of Eq.~(\ref{cmmAssumption}), while Eq.~(\ref{cmmAssumption}) is clearly stricter than the properness assumption~\eqref{propercond} since it removes the dependence of $C$ on the instance by marginalizing $\boldsymbol{x}$.
Therefore, the CMM assumption~\eqref{cmmAssumption} is stronger than the properness assumption~\eqref{propercond}. 
\begin{table}[t]\centering
\caption{Examples of special PPL settings.}
\vspace{0.2cm}
\setlength{\tabcolsep}{5mm}{
        \begin{tabular}{ c|c c c c} 
        \hline
          & CMM & CL & MCL & PCPL \\ 
        \hline
        $C(\boldsymbol{x},s)$ & $C(s)$ & $\frac{1}{K-1}$ & $\frac{Q_{|s|}}{\binom{K-1}{|s| - 1}}$ & $\frac{1}{2^{K-1}-1}$\\
        \hline
        \end{tabular}}
        
        \label{table:summary}
    \end{table} 
Within the PL framework, we can reproduce the CL, MCL, and PCPL models introduced in Section 2 with different formulations of $C(\boldsymbol{x}, s)$ (Table~\ref{table:summary}).
We notice that CL, MCL, and PCPL have stronger assumptions than PPL for the same reason: they all assume that given a specific true label, the candidate label set is independent of the instance, i.e., mapping function $C$ is independent of $\boldsymbol{x}$. Thus they are the special cases of PPL. 
Besides, for CL classification, we have
\begin{equation*}
    Q_{|s|} = 
    \begin{cases}
    1 &\text{if }|s|=K-1,\\
    0 &\text{otherwise.}
    \end{cases}
\end{equation*}
This shows that the CL setting is stricter than the MCL setting. The following proposition shows that the PCPL assumption is stronger than the MCL assumption:
\begin{proposition}\label{proppcplstronger}
If the PCPL assumption $
    p(s|\boldsymbol{x}, y) = \frac{1}{2^{K-1}-1}\boldsymbol{1}\{y\in s\}
$ holds, then the MCL assumption $    p(s|\boldsymbol{x}, y) = \frac{Q_{|s|}}{\binom{K-1}{|s| - 1}}\boldsymbol{1}\{y\in s\}
$ also holds. But the opposite is not true.
\end{proposition}
    \begin{figure}[t]
    \begin{center}
        \includegraphics[scale=0.5]{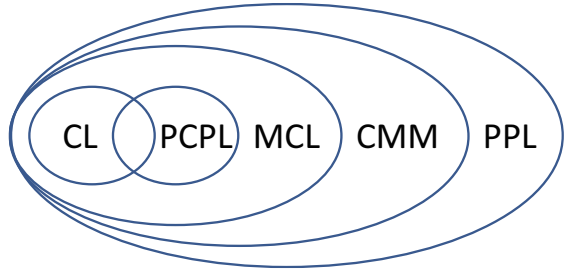}
        \caption{Diagram representing the assumption relations among different PPL settings.}
        \label{figure:summary}
    \end{center}        
    \end{figure}
We summarize the relations between different special PPL settings in Figure~\ref{figure:summary}. The area AND(CL, PCPL) represents MCL problems where $K = 2$. The partial label generation process adopted in Section 4 lies in the area MCL$\setminus$OR(CL, PCPL). The area PPL$\setminus$CMM represents PPL problems where the function $C$ also depends on $\boldsymbol{x}$. The area outside of PPL represents PL problems where $p(s|\boldsymbol{x}, y)$ depends on the choice of $y$ given $y\in s$.

\subsection{Risk-Consistent Algorithm}
Here, we describe our proposed algorithm for PPL.
\paragraph{Candidate Label Confidence.}
Various weakly supervised learning settings based on confidence data have been studied recently, including positive-confidence data~\citep{ishida2018binary}, similarity-confidence data~\citep{cao2021learning}, confidence data for instance-dependent label-noise~\citep{berthon2021confidence} and single-class confidence~\citep{cao2021multi}. Here, we define label confidence, signifying the label posterior probability given the instance and partial label:
\begin{definition}\label{def: candidateconf}
The candidate label confidence $r_{y}(\boldsymbol{x}, s)$ is defined as 
\begin{equation*}
    r_{y}(\boldsymbol{x}, s) = p(y|\boldsymbol{x}, s).
\end{equation*}
\end{definition}

Then we have the following theorem which provides an equivalent definition for the properness of PL:
\begin{theorem}\label{theorem:pplequivalence}
The properness assumption~\eqref{propercond} holds if and only if
\begin{equation}\label{pplConfidence}
    r_y(\boldsymbol{x}, s) = \frac{p(y|\boldsymbol{x})}{\sum_{k\in s}p(y=k|\boldsymbol{x})}\boldsymbol{1}\{y\in s\}.
\end{equation}
\end{theorem}
Therefore, the properness implies that given a set of candidate labels, the label confidence of one of the candidates is proportional to its class posterior probability.
Conversely, if the label confidence can be expressed in the above equation, the PL problem is properness.
The advantage of this property is that we can compute $r_y(\boldsymbol{x}, s)$ without specifying $C(\boldsymbol{x}, s)$, which enables our derivation for a risk-consistent algorithm.

\paragraph{Unbiased Risk Estimator.}

Here, we will use the PPL training dataset ${\mathcal{D}}_\text{PPL}=\{(\boldsymbol{x}_i, s_i)\}_{i=1}^n$ to construct an unbiased estimator of the objective risk~\eqref{eq:risk}.
\begin{theorem}\label{theorem:clc}
   The classification risk~\eqref{eq:risk} can be equivalently expressed as
   \begin{equation*}
       R(f;{\mathcal{L}}) = \mathbb{E}_{p(\boldsymbol{x}, s)}\left[\sum_{j\in s}r_j(\boldsymbol{x}, s){\mathcal{L}}(f(\boldsymbol{x}), j)\right].
   \end{equation*}
\end{theorem}
Theorem~\ref{theorem:clc} indicates the existence of label confidence that was defined in Eq.~(\ref{pplConfidence}) for PL by which the weighted loss is risk-consistent.
Based on Theorem~\ref{theorem:pplequivalence} and Theorem~\ref{theorem:clc}, we can have the following theorem:
   \begin{theorem}\label{theorem:ppl}
   For PPL, the classification risk~\eqref{eq:risk} can be equivalently expressed as
   \begin{equation}\label{eq:pplrisk}
       R(f;{\mathcal{L}}) = \mathbb{E}_{p(\boldsymbol{x}, s)}\left[\sum_{j\in s}\frac{p(y=j|\boldsymbol{x})}{\sum_{k\in s}p(y=k|\boldsymbol{x})}{\mathcal{L}}(f(\boldsymbol{x}), j)\right].
   \end{equation}
   \end{theorem}
   Finally, an unbiased risk estimator is given by the following corollary:
   \begin{corollary}
   When the properness assumption~\eqref{propercond} holds, given the dataset ${\mathcal{D}}_\mathrm{PPL}=\{(\boldsymbol{x}_i, s_i)\}_{i=1}^n$, the following is an unbiased estimator of the classification risk:
   \begin{equation}\label{emprisk:ppl}
       \hat{R}_\mathrm{PPL}(f;{\mathcal{L}}) = \frac{1}{n}\sum_{i=1}^n\sum_{j\in s_i}\frac{p(y=j|\boldsymbol{x}_i)}{\sum_{k\in s_i}p(y=k|\boldsymbol{x}_i)}{\mathcal{L}}(f(\boldsymbol{x}_i), j).
   \end{equation}
   \end{corollary}
\paragraph{Risk-Consistent Algorithm.}
\begin{algorithm}[t]
    \caption{PPL}
    \hspace*{0.02in} {\bf Input:}
    Model $f$, number of epochs $T$, number of mini-batches $B$, partially labeled training dataset ${\mathcal{D}}_\mathrm{PPL}=\{(\boldsymbol{x}_i, s_i)\}_{i=1}^n$
    \begin{algorithmic}[1]
    \State \textbf{Initialize} $\hat{p}(y=j|\boldsymbol{x}_i)=1\mathbin{/}K$;
    \State Let $\mathcal{A}$ be an stochastic optimization algorithm;
    \For{$t = 1, 2,\ldots,T$}
    \State \textbf{Shuffle} ${\mathcal{D}}_\mathrm{PPL}$;
        \For{$l = 1, 2,\ldots, B$}
            \State \textbf{Fetch} mini-batch ${\mathcal{D}}_l$ from ${\mathcal{D}}_\mathrm{PPL}$;
            
            \State \textbf{Compute} loss from Eq.~\eqref{emprisk:ppl} using $\hat{p}$;
            \State \textbf{Update} $f$ by $\mathcal{A}$;
            \State \textbf{Update} $\hat{p}$ by Eq.~\eqref{softmaxRevisited};
        \EndFor
        \State \textbf{end for}
    \EndFor
    \State \textbf{end for}
    \end{algorithmic}
    \hspace*{0.02in} {\bf Output: $f$}
    \label{alg_ppl}
    \end{algorithm}
We approximate the class-posterior probability $p(y=j|\boldsymbol{x})$ by the softmax output:
\begin{equation}\label{softmaxRevisited}
    p_f(y|\boldsymbol{x}) = \frac{e^{f_y(\boldsymbol{x})}}{\sum_{i=1}^Ke^{f_i(\boldsymbol{x})}}.
\end{equation}
Then we can compute the empirical risk~\eqref{emprisk:ppl} with ${\mathcal{D}}_\mathrm{PPL}$ and~\eqref{softmaxRevisited}. 
The label confidences are estimated by a progressive method, similar to \cite{lv2020progressive}, that is, the update of the model and the estimation of the label confidences are accomplished seamlessly. Comparing with the classical EM methods that trains the model until convergence in every M-step, such progressive approach gets rid of the overfitting issues since it does not overemphasize the convergence, so that the model will not fit the initial inexact prior knowledge, thereby reducing subsequent estimation error.
The pseudo-code and the implementation of the proposed method are presented in Algorithm~\ref{alg_ppl} and Figure~\ref{fig_ppl}. 
We can see that the proposed method allows the use of any model (ranging from linear to deep model), is compatible with a large group of loss functions and stochastic optimization.
Note that our practical implementation coincides with that of the \emph{progressive identification} (PRODEN) algorithm~\citep{lv2020progressive} and the PCPL algorithm~\citep{feng2020provably}. 
However, they were derived in totally different manners.
To be specific, PCPL is based on a uniform data generation assumption~(\ref{eq:uniform}), PPL proposes a unified framework that encapsulates PCPL and several classical assumptions in PL community, and PRODEN derives from an intuitive insight that the true label will incur the minimal loss.
As a result, PRODEN can only be proved to be classifier-consistent under the deterministic scenario, PCPL is risk-consistent under a strong uniform sampling process assumption, and our result on risk-consistency only requires that the PL problem is proper and applies to both deterministic and stochastic scenarios.

\begin{figure}[t]
	\begin{center}
		\includegraphics[scale=0.45]{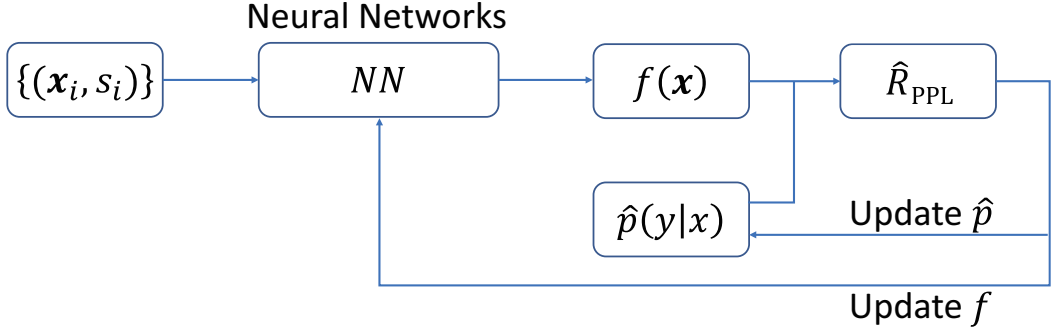}
	\end{center}
	\caption{Implementation diagram of PPL.}
	\label{fig_ppl}
\end{figure}

Next, we establish an estimation error bound to analyze the estimation error $R({\hat{f}_\mathrm{PPL}}) - R(f^*)$, where $f^*=\arg\min_{f\in{\mathcal{F}}} R(f)$ is the optimal classifier and $\hat{f}_{\mathrm{PPL}}=\arg\min_{f\in\mathcal{F}}\hat{R}_{\mathrm{PPL}}(f)$ is the empirical risk minimizer.
To this end, we need to define a class of real functions $\mathcal{F}_i$, and then $\mathcal{F}=\oplus_{i\in[K]}\mathcal{F}_i$ is the $K$-valued function space. 
Then we have the following theorem:
\begin{theorem}\label{theorem:bound}
Suppose that ${\mathcal{L}}(f(\boldsymbol{x}), s)$ is $\rho$-Lipschitz with respect to $f(\boldsymbol{x})$ for all $y\in{\mathcal{Y}}$ and upper-bounded by $M$, i.e., $M=\sup_{\boldsymbol{x}\in{\mathcal{X}},y\in{\mathcal{Y}},f\in{\mathcal{F}}}{\mathcal{L}}(f(\boldsymbol{x}),y)$. Let $\mathfrak{R}_n({\mathcal{F}}_i)$ be the Rademacher complexity \citep{mohri2018foundations} of $\mathcal{F}_i$ with sample size $n$. Then, for any $\delta>0$, with probability at least $1-\delta$,
\begin{equation*}
R({\hat{f}_\mathrm{PPL}}) - R(f^*) \leq 4\sqrt{2}\rho\sum_{i=1}^K \mathfrak{R}_n({\mathcal{F}}_i)+ 2M\sqrt{\frac{\log\frac{2}{\delta}}{2n}}.
\end{equation*}
\end{theorem}
As $n\rightarrow\infty$, $\mathfrak{R}_n({\mathcal{F}}_i)\rightarrow 0$ for all parametric models with a bounded norm such as deep networks trained with weight decay \citep{lu2019on}.
Hence, the above theorem demonstrates that $\hat{f}_\mathrm{PPL}$ converges $f^*$ as the number of training data tends to infinity, i.e., the learning is consistent.

\section{Experiments}
In this section, we empirically demonstrate the effectiveness of the proposed method PPL. 
The implementation is based on PyTorch and experiments were carried out with NVIDIA Tesla V100 GPU.
Experiment details are provided in Appendix~\ref{Appendix_Expdetails}.

\paragraph{Datasets.}
Our experiments were conducted on four widely-used benchmark datasets in image classification, namely, MNIST \citep{lecun1998gradient}, Fashion-MNIST \citep{xiao2017fashion}, Kuzushiji-MNIST \citep{clanuwat2018deep}, and CIFAR-10 \citep{Krizhevsky2009cifar}.
We manually corrupted these datasets into partially labeled versions by using generated  partial labels in the area MCL$\setminus$OR(CL, PCPL) (Figure~\ref{figure:summary}).

\paragraph{Model and Loss Function. }
For showing that the superiority of the proposed method is not caused by a specific model, we conducted experiments with various base models, including a simple linear model ($d$-10), a 5-layer multi-layer perceptron (MLP) ($d$-300-300-300-300-10), a 32-layer ResNet~\citep{he2016deep}, and a 22-layer DenseNet~\citep{huang2017densely}. 
We applied batch normalization~\citep{ioffe2015batch} to each hidden layer and used $l_2$-regularization.
We adopted the popular cross-entropy loss as the loss function,
and used the stochastic gradient descent optimizer~\citep{robbins1951stochastic} with momentum $0.9$.
 
 \begin{figure}[t]
	\begin{center}
		\includegraphics[scale=0.9]{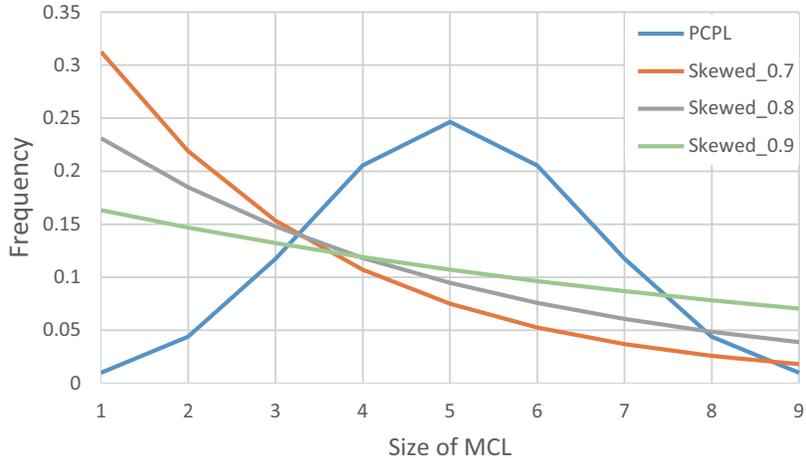}
		\caption{Illustrative example of different distributions.}
		\label{skewandmcl}
	\end{center}        
\end{figure}

\begin{table}[t]\centering
	\caption{Average partial label size for different distributions.}
	\vspace{0.2cm}
	\newcommand{\tabincell}[2]{\begin{tabular}{@{}#1@{}}#2\end{tabular}}
	\setlength{\tabcolsep}{5mm}{
		\begin{tabular}{ c|c c c c } 
			\hline
			& \tabincell{c}{PCPL} & \tabincell{c}{0.9-skewed} & \tabincell{c}{0.8-skewed} & 
			\tabincell{c}{0.7-skewed} \\ 
			\hline
			$\mathbb{E}\left[|s|\right]$ & 5.0 & 5.69 & 6.40  & 7.05 \\ 
			
			\hline
	\end{tabular}}
	
	\label{averagesize}
\end{table} 

\begin{figure*}[!t]
	\vskip 0.2in
	\subfigure{
		\begin{minipage}[b]{0.24\columnwidth}
			\centering
			\includegraphics[width=1.5in]{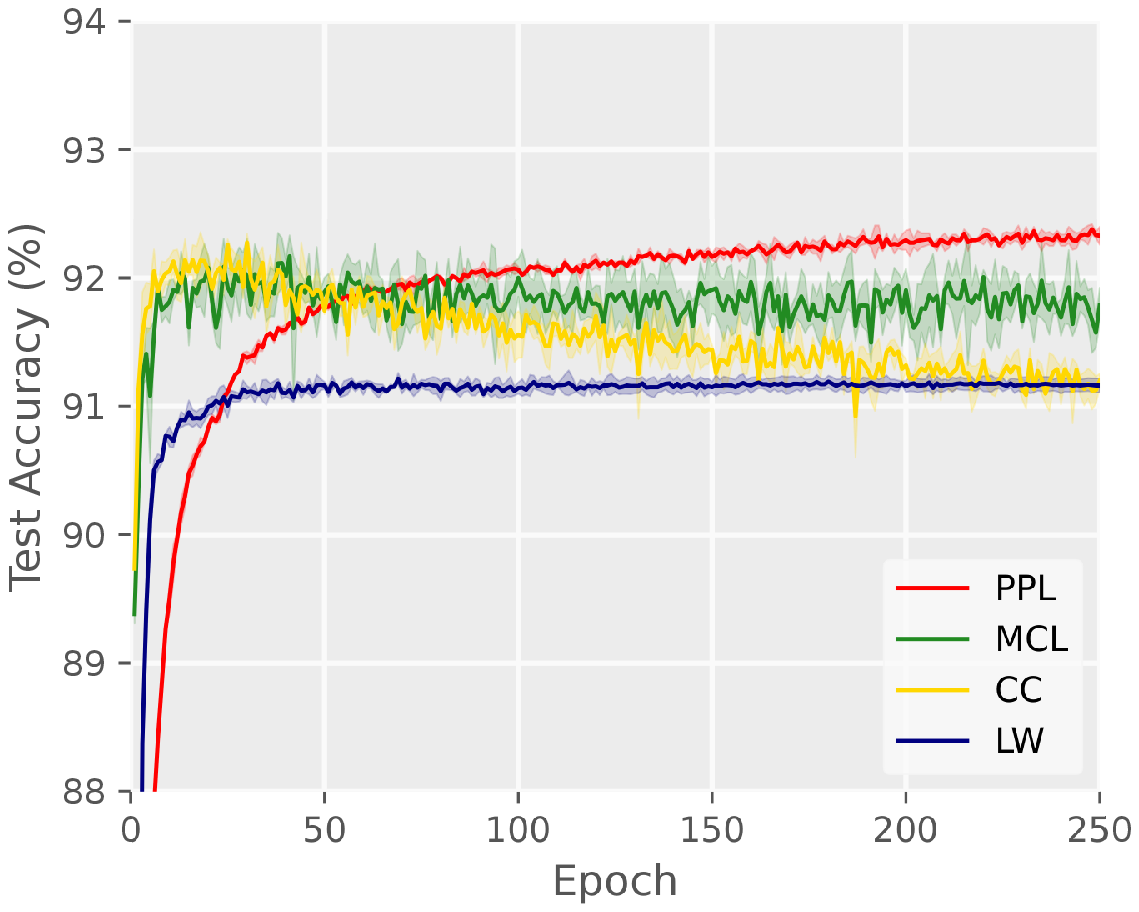}
			\centerline{\quad MNIST, $\alpha=0.9$}
	\end{minipage}}
	\subfigure{
		\begin{minipage}[b]{0.24\columnwidth}
			\centering
			\includegraphics[width=1.5in]{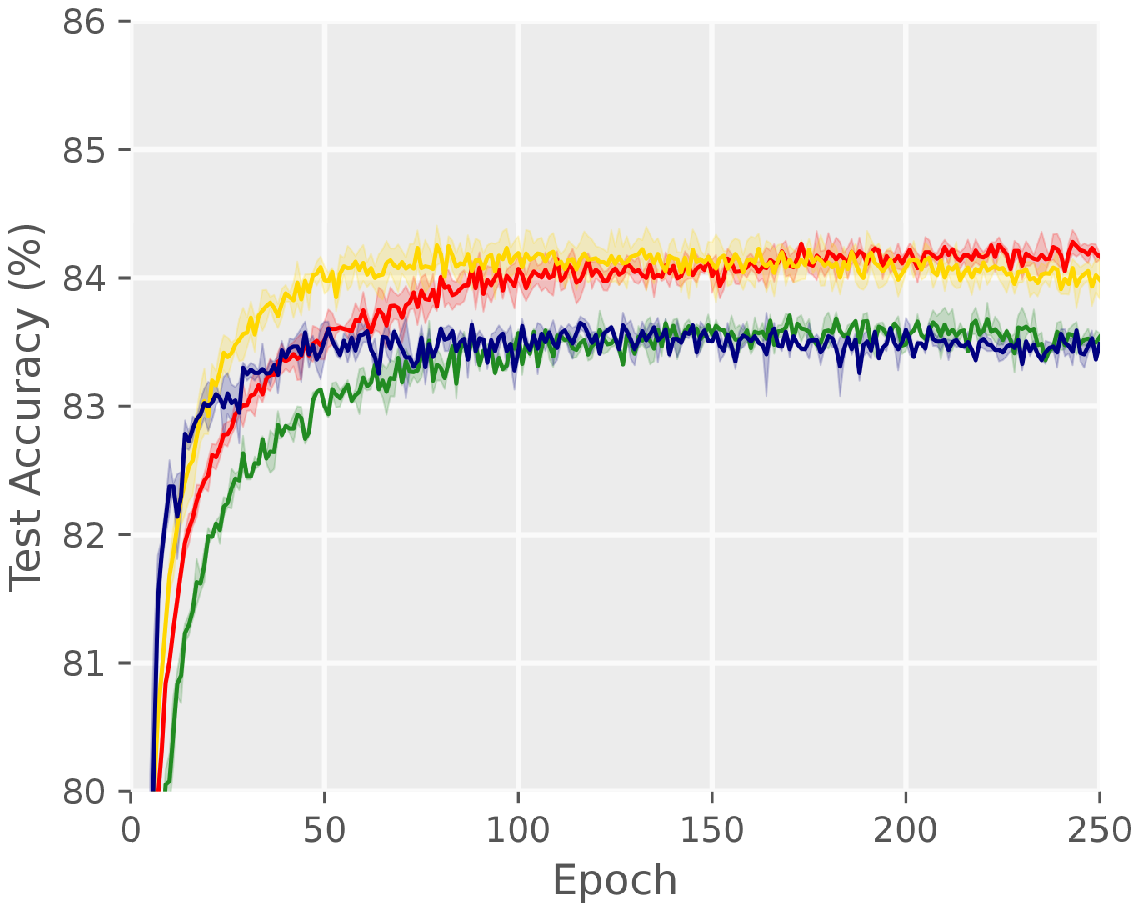}
			\centerline{\quad Fashion, $\alpha=0.9$}
	\end{minipage}}
	\subfigure{
		\begin{minipage}[b]{0.23\columnwidth}
			\centering
			\includegraphics[width=1.5in]{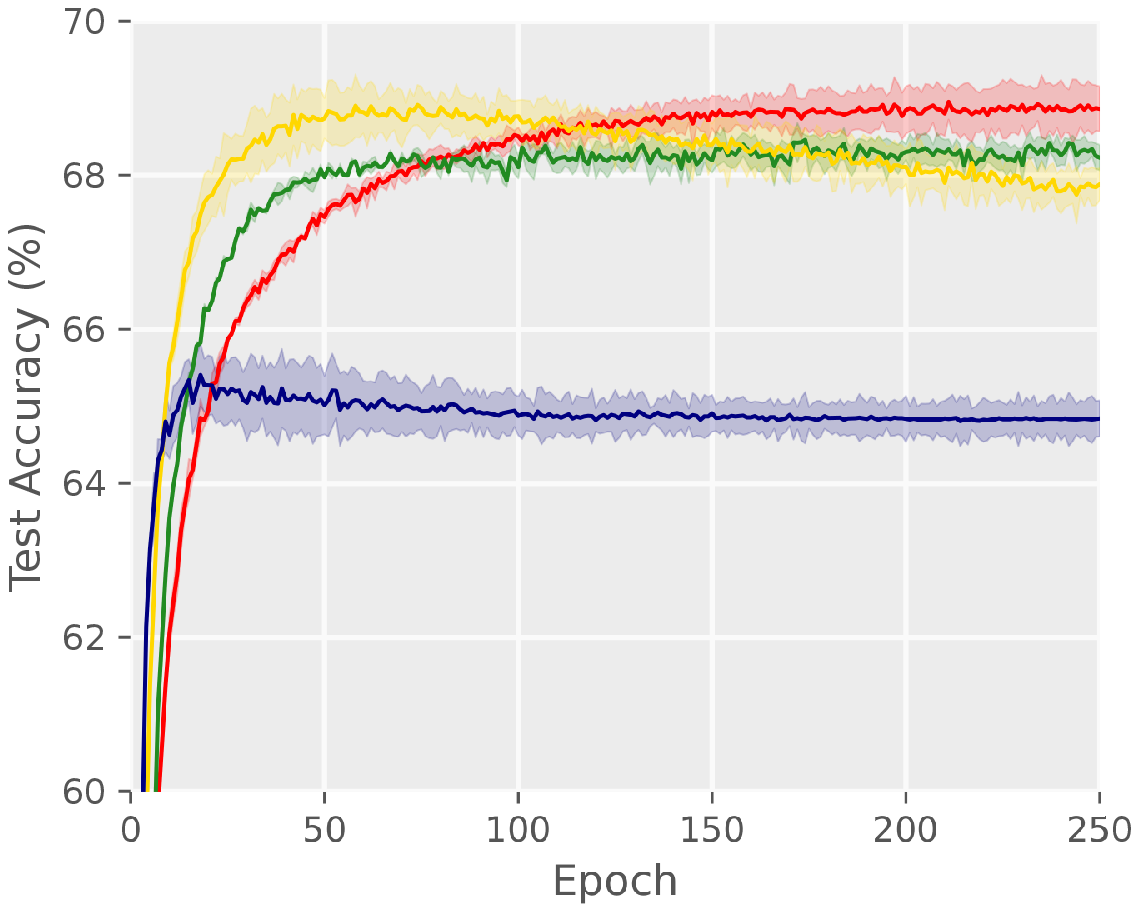}
			\centerline{\quad Kuzushiji, $\alpha=0.9$}
	\end{minipage}}
	\subfigure{
		\begin{minipage}[b]{0.23\columnwidth}
			\centering
			\includegraphics[width=1.5in]{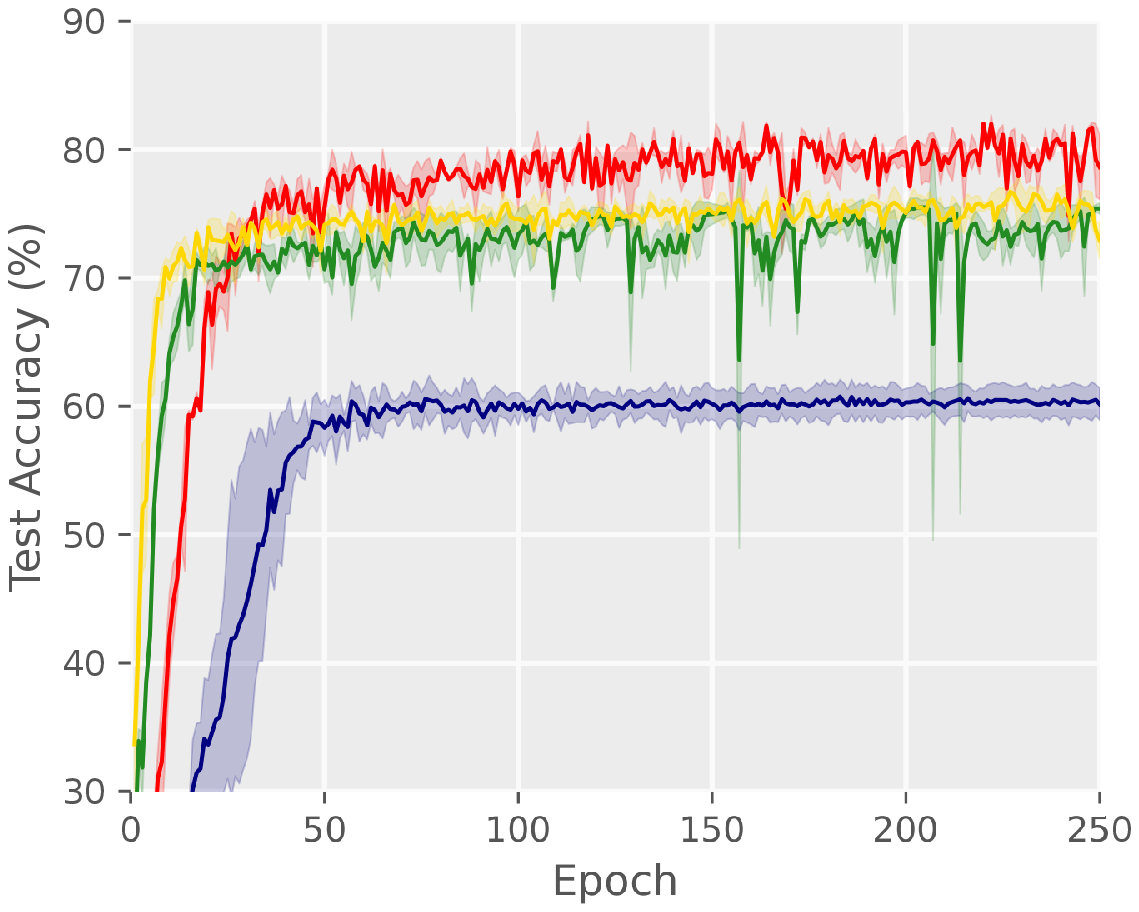}
			\centerline{\quad CIFAR-10, $\alpha=0.9$}
	\end{minipage}}
	
	\subfigure{
		\begin{minipage}[b]{0.24\columnwidth}
			\centering
			\includegraphics[width=1.5in]{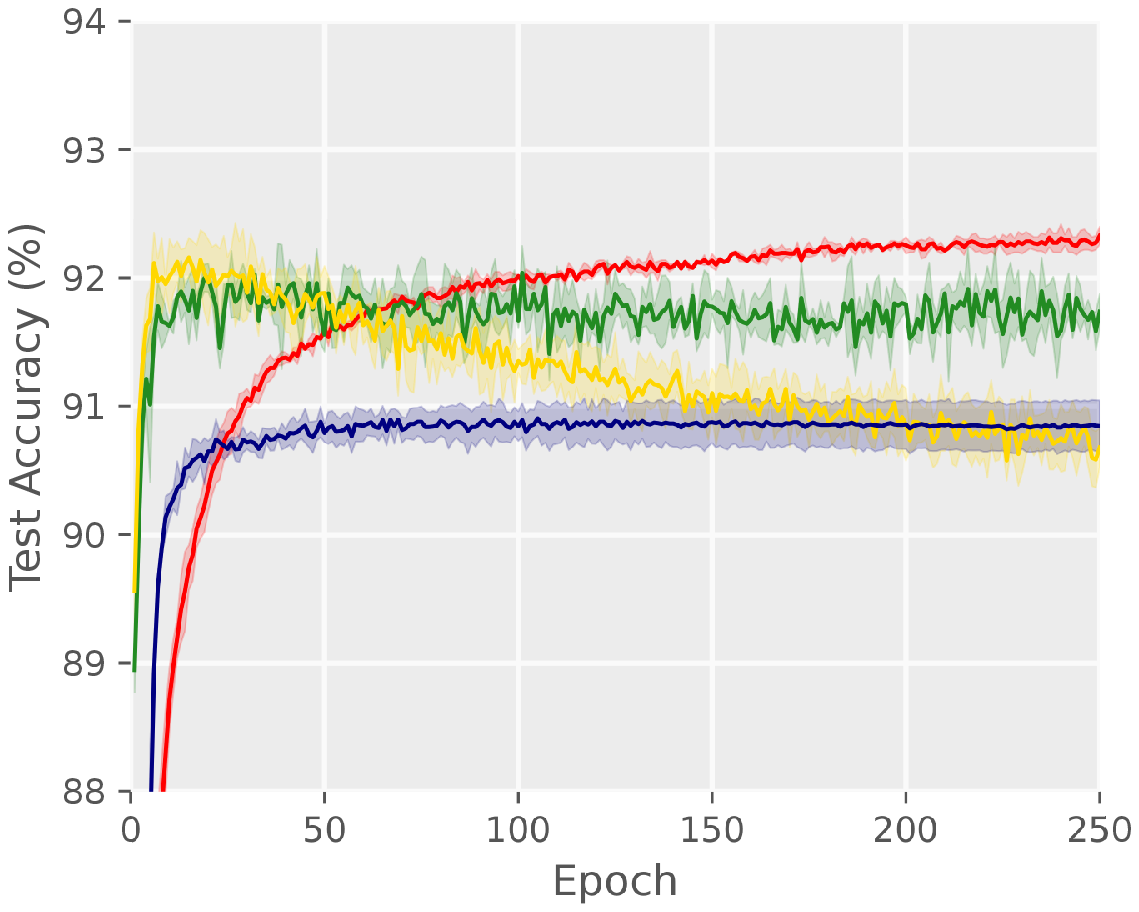}
			\centerline{\quad MNIST, $\alpha=0.8$}
	\end{minipage}}
	\subfigure{
		\begin{minipage}[b]{0.24\columnwidth}
			\centering
			\includegraphics[width=1.5in]{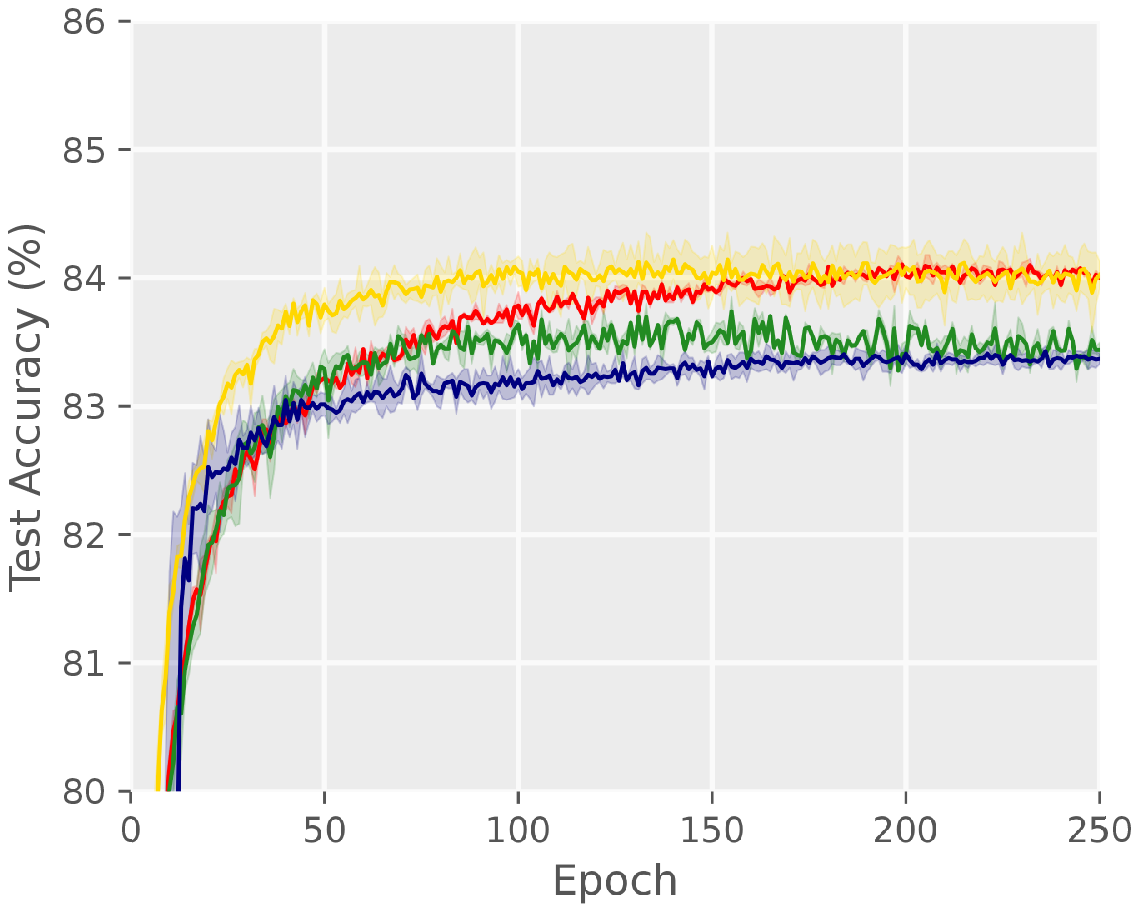}
			\centerline{\quad Fashion, $\alpha=0.8$}
	\end{minipage}}
	\subfigure{
		\begin{minipage}[b]{0.23\columnwidth}
			\centering
			\includegraphics[width=1.5in]{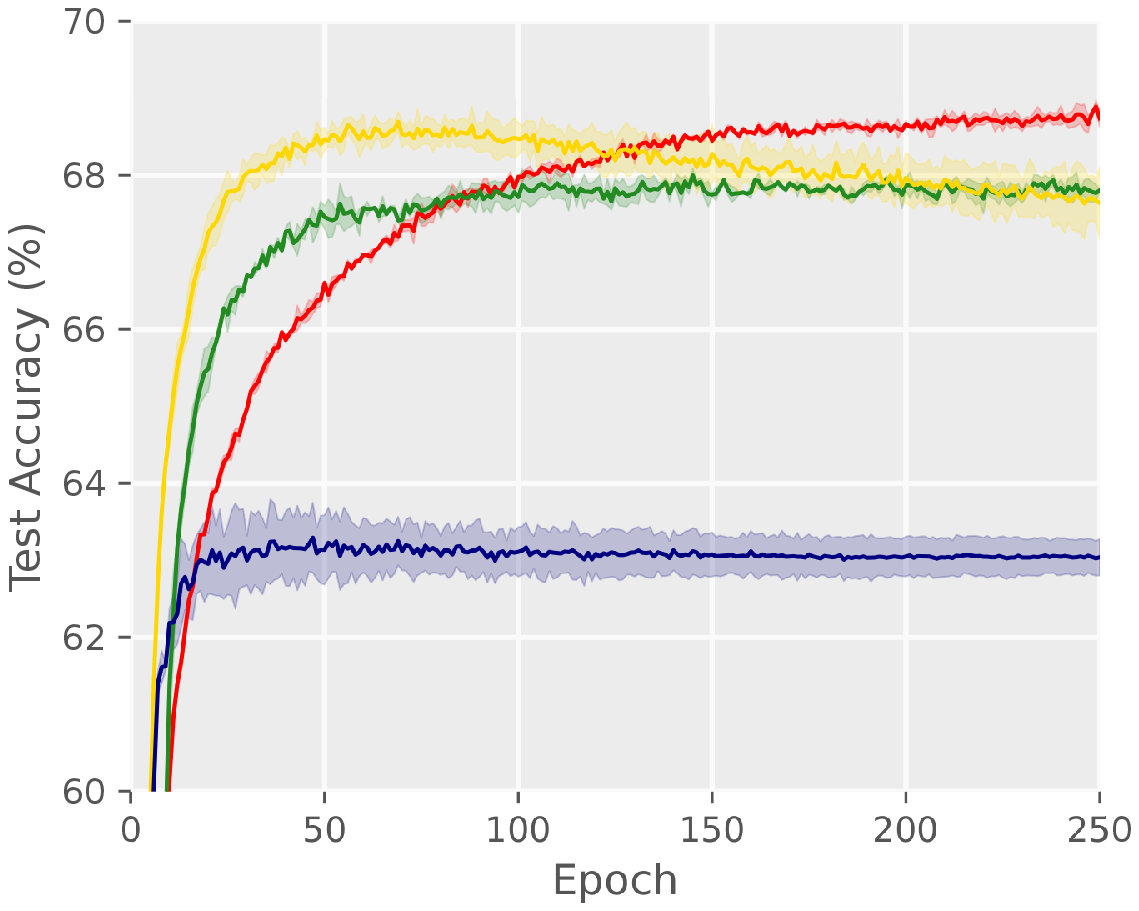}
			\centerline{\quad Kuzushiji, $\alpha=0.8$}
	\end{minipage}}
	\subfigure{
		\begin{minipage}[b]{0.23\columnwidth}
			\centering
			\includegraphics[width=1.5in]{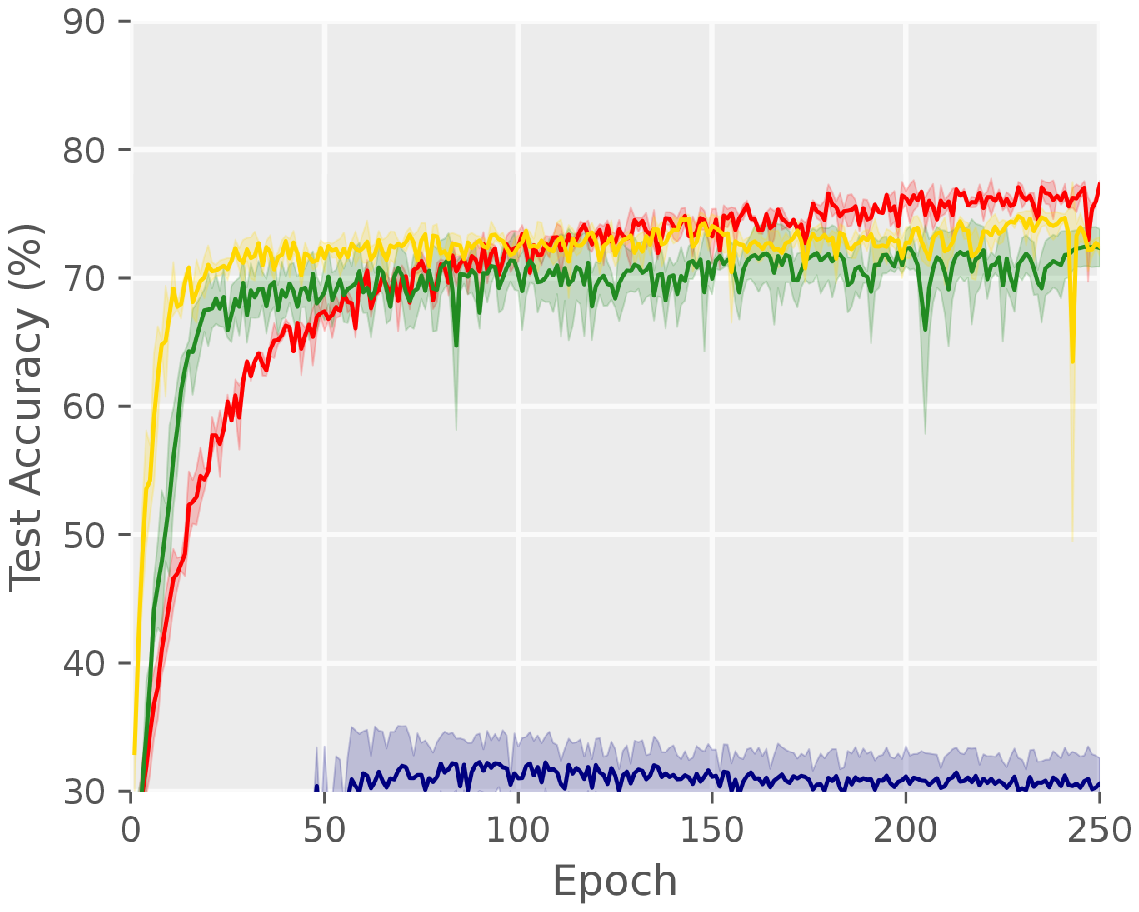}
			\centerline{\quad CIFAR-10, $\alpha=0.8$}
	\end{minipage}}
	
	\subfigure{
		\begin{minipage}[b]{0.24\columnwidth}
			\centering
			\includegraphics[width=1.5in]{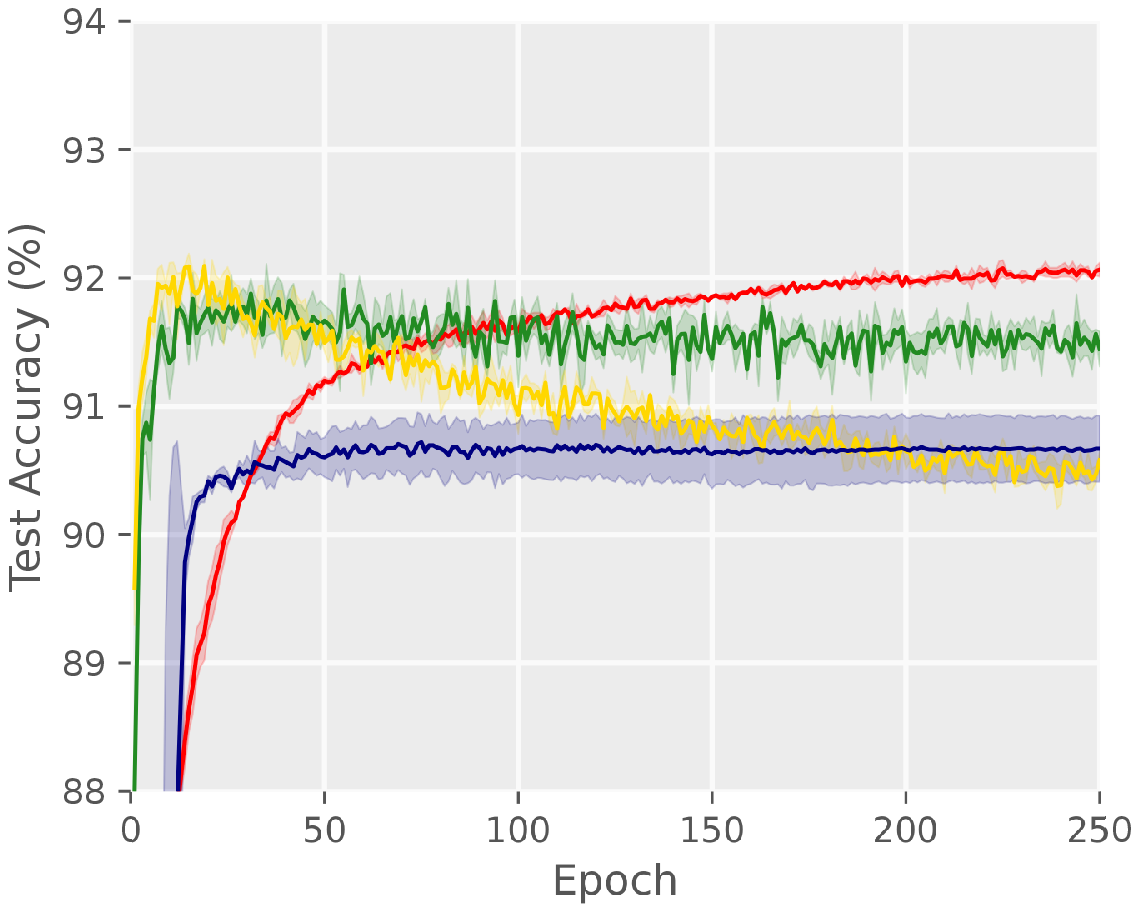}
			\centerline{\quad MNIST, $\alpha=0.7$}
	\end{minipage}}
	\subfigure{
		\begin{minipage}[b]{0.24\columnwidth}
			\centering
			\includegraphics[width=1.5in]{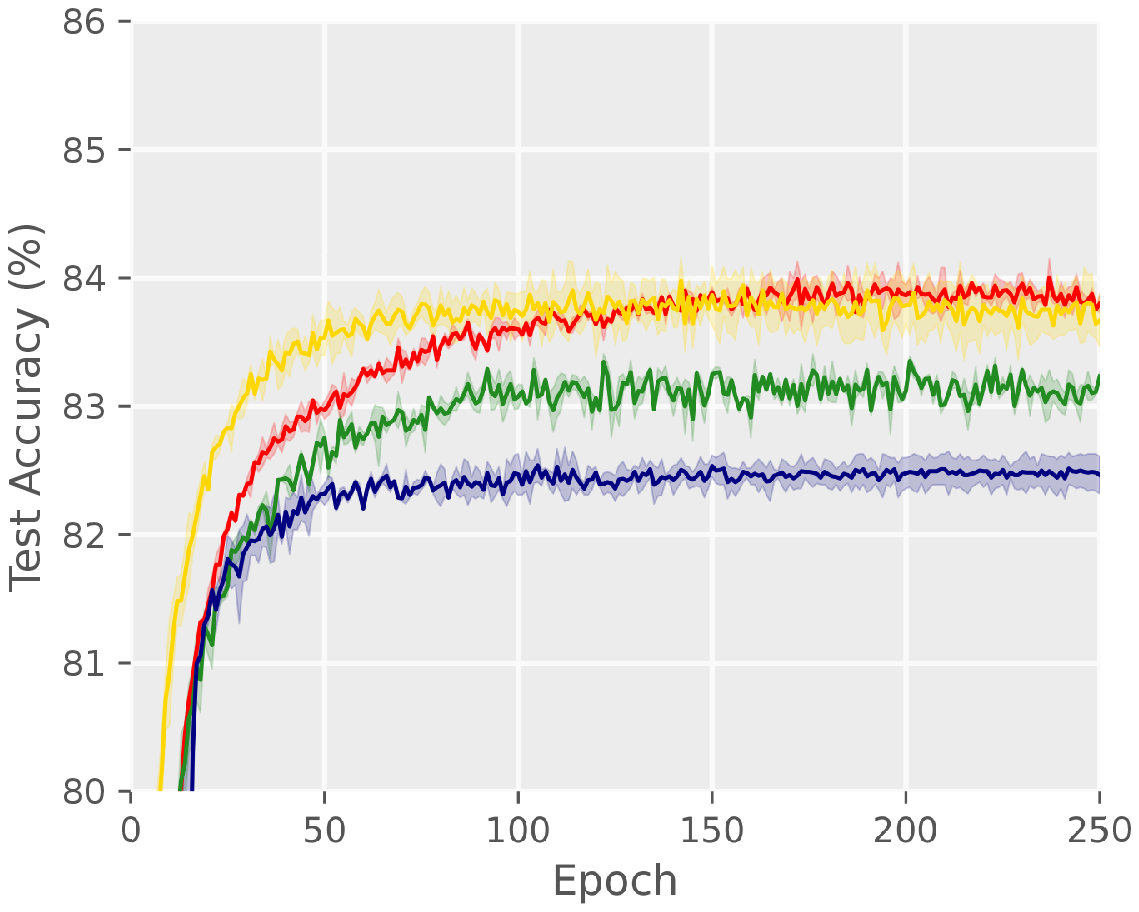}
			\centerline{\quad Fashion, $\alpha=0.7$}
	\end{minipage}}
	\subfigure{
		\begin{minipage}[b]{0.23\columnwidth}
			\centering
			\includegraphics[width=1.5in]{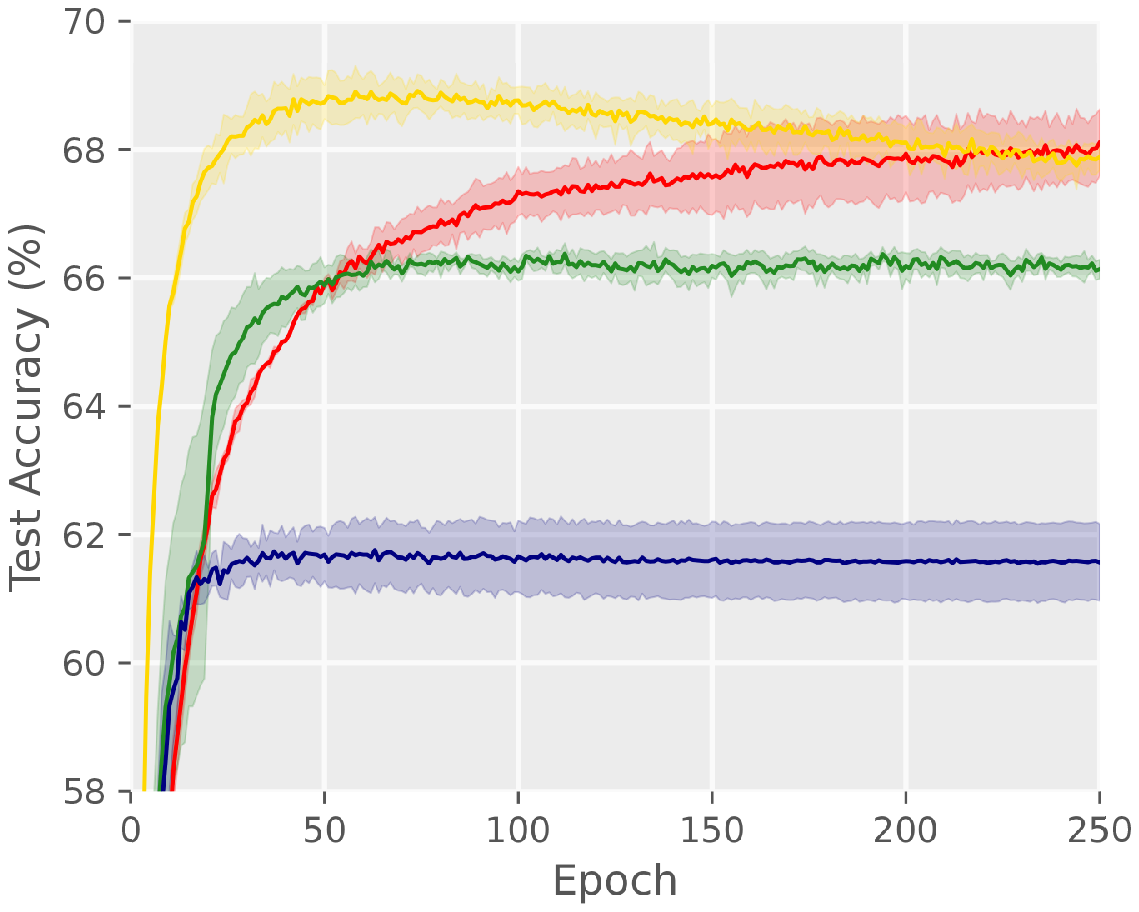}
			\centerline{\quad Kuzushiji, $\alpha=0.7$}
	\end{minipage}}
	\subfigure{
		\begin{minipage}[b]{0.23\columnwidth}
			\centering
			\includegraphics[width=1.5in]{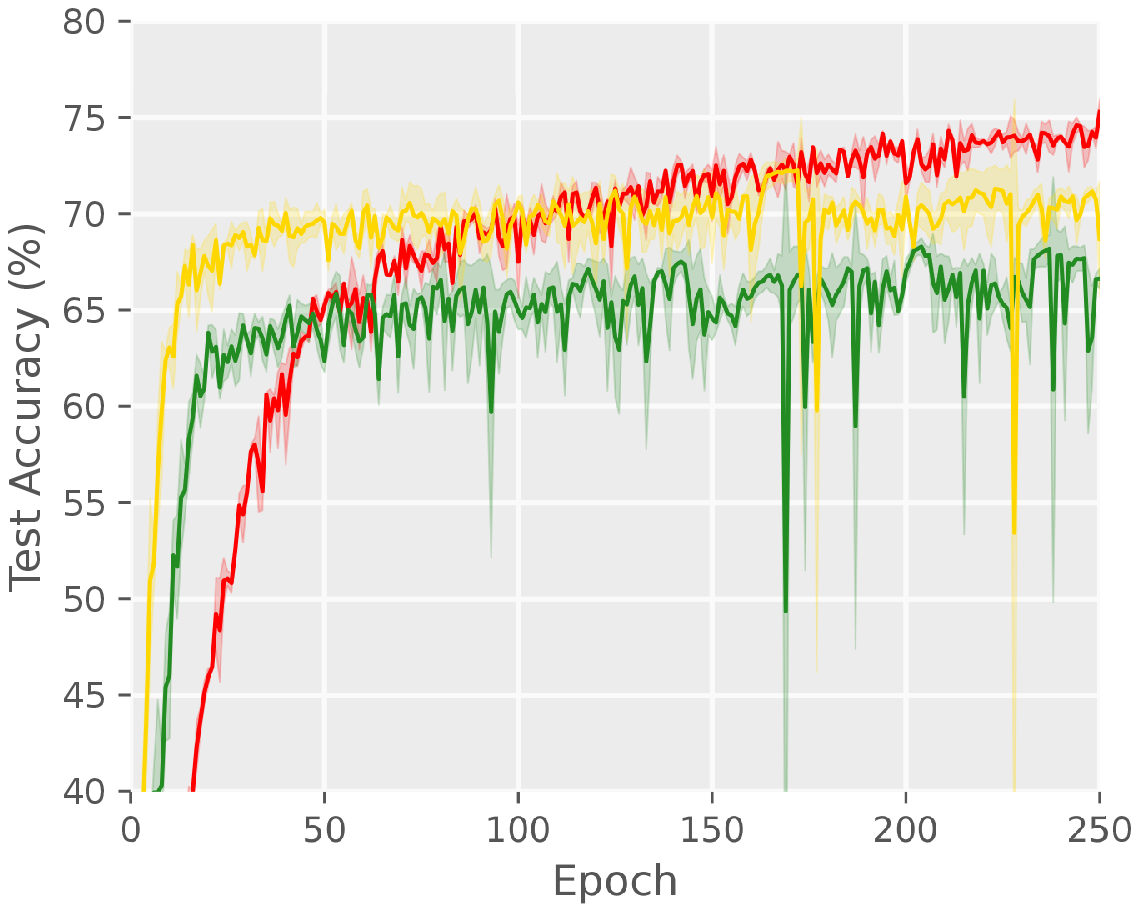}
			\centerline{\quad CIFAR-10, $\alpha=0.7$}
	\end{minipage}}
	
	\caption{Test accuracy with different $\alpha$. The linear model was trained on MNIST, Fashion-MNIST, and Kuzushiji-MNIST, and ResNet was trained on CIFAR-10. Dark colors represent the mean accuracy of 5 trials and light colors represent the standard error.}
	\label{fres1}
\end{figure*}

\begin{figure*}[!t]
	\vskip 0.2in
	\subfigure{
		\begin{minipage}[b]{0.25\columnwidth}
			\centering
			\includegraphics[width=1.52in]{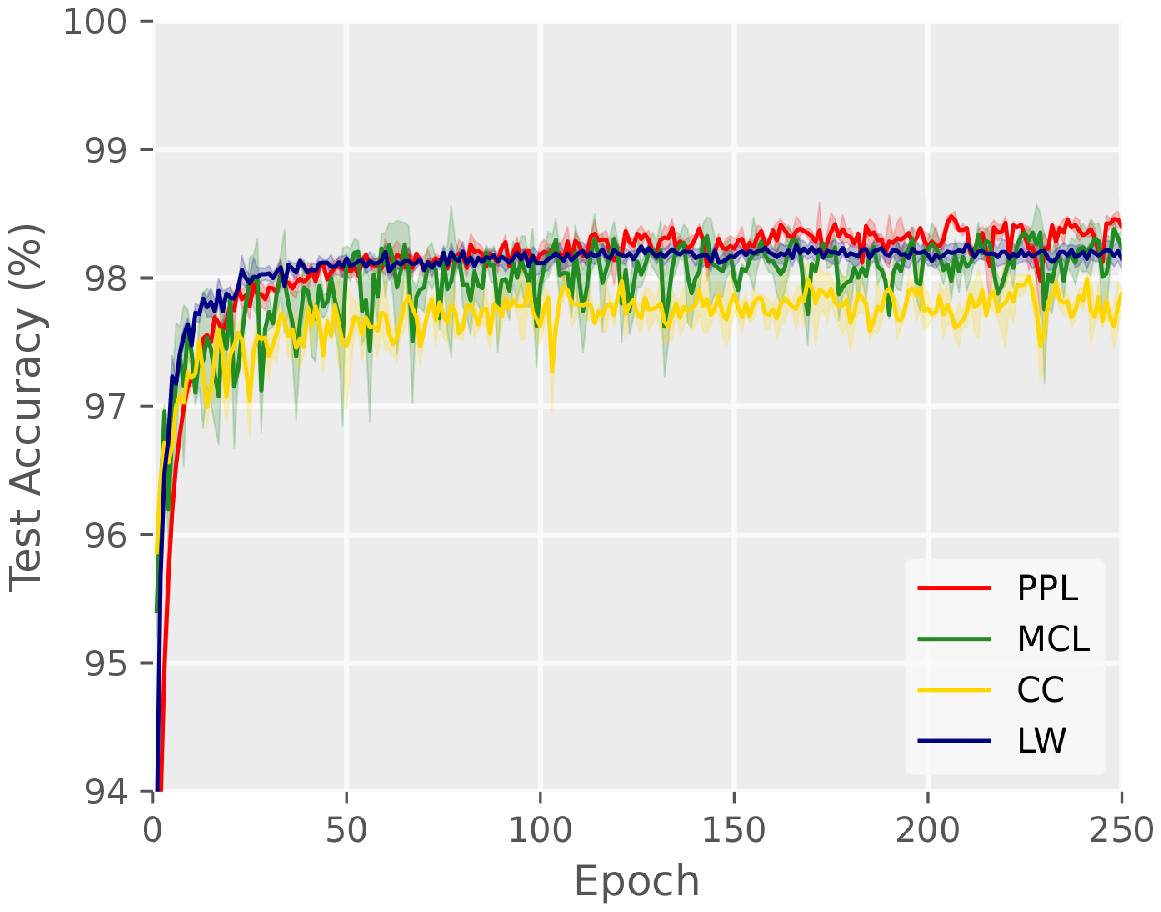}
			\centerline{\quad MNIST, $\alpha=0.9$}
	\end{minipage}}
	\subfigure{
		\begin{minipage}[b]{0.24\columnwidth}
			\centering
			\includegraphics[width=1.5in]{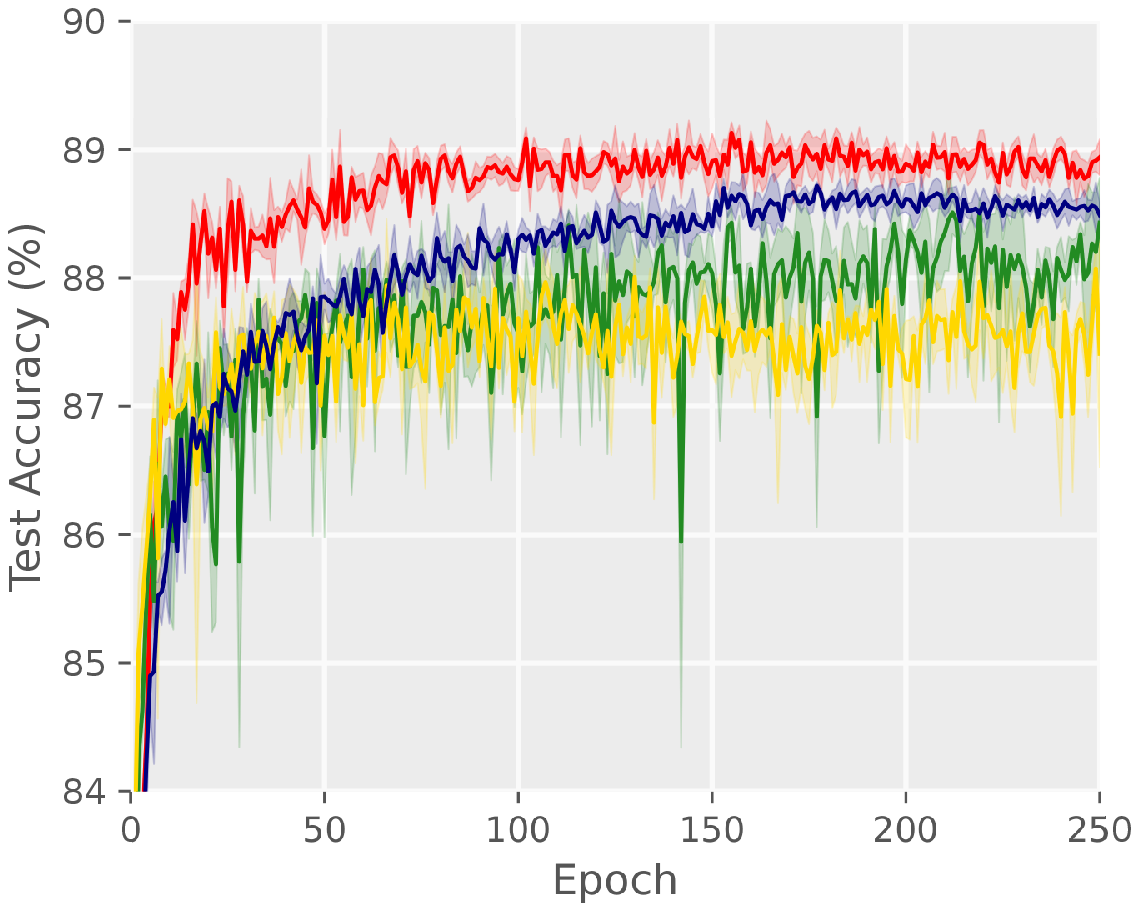}
			\centerline{\quad Fashion, $\alpha=0.9$}
	\end{minipage}}
	\subfigure{
		\begin{minipage}[b]{0.23\columnwidth}
			\centering
			\includegraphics[width=1.5in]{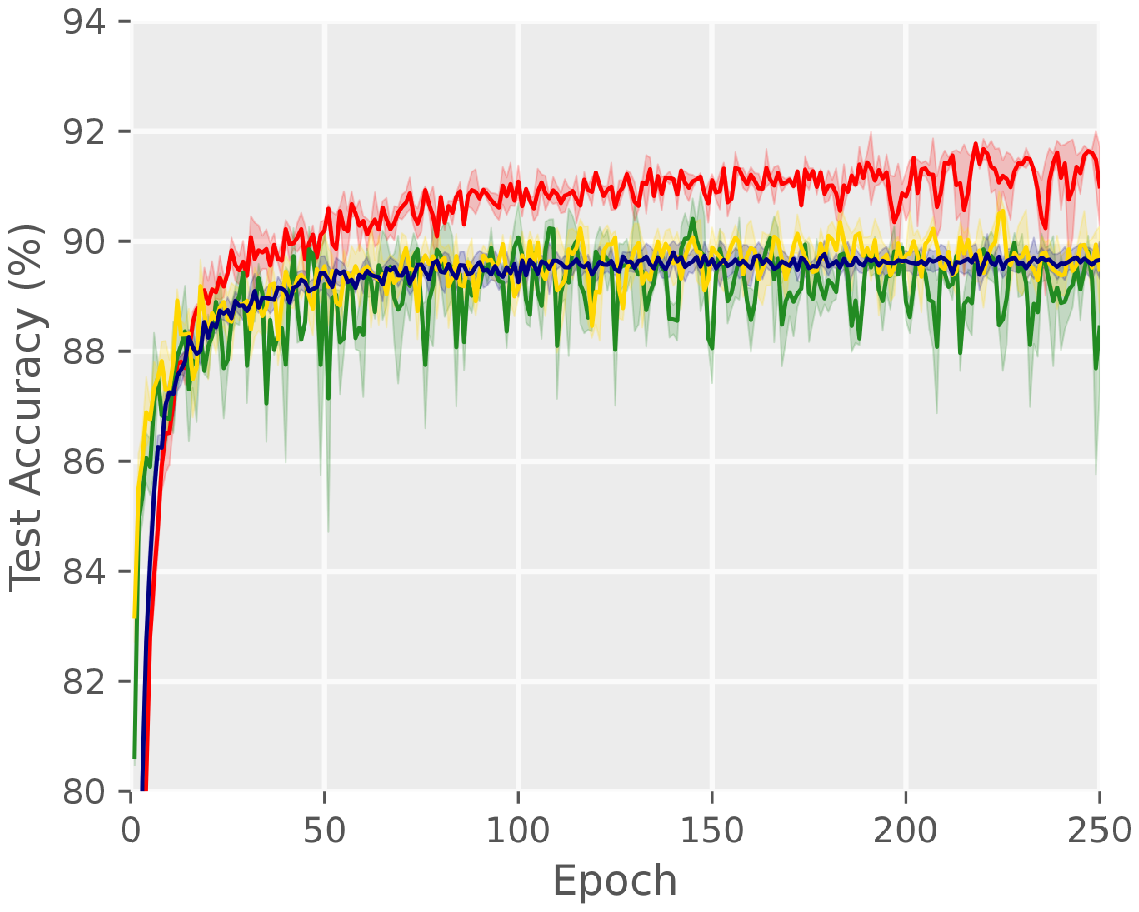}
			\centerline{\quad Kuzushiji, $\alpha=0.9$}
	\end{minipage}}
	\subfigure{
		\begin{minipage}[b]{0.23\columnwidth}
			\centering
			\includegraphics[width=1.5in]{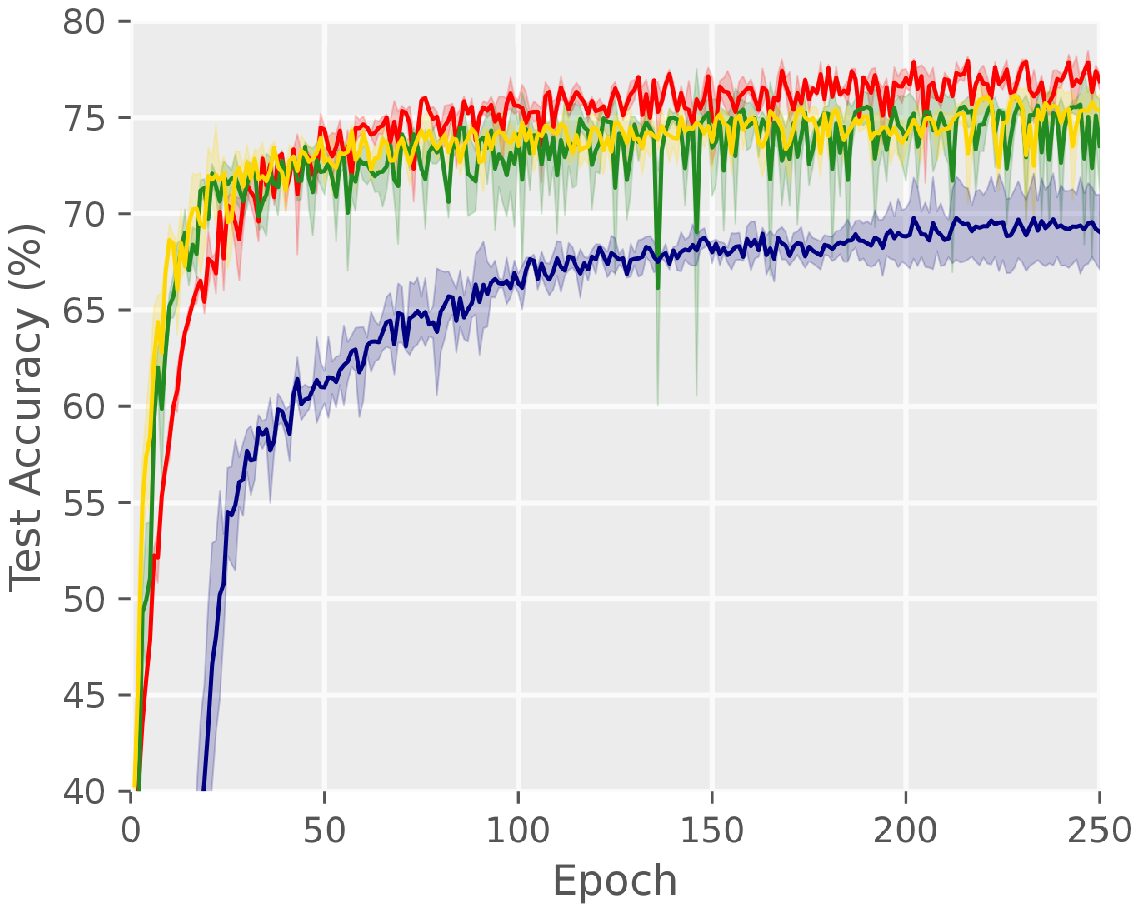}
			\centerline{\quad CIFAR-10, $\alpha=0.9$}
	\end{minipage}}
	
	\subfigure{
		\begin{minipage}[b]{0.25\columnwidth}
			\centering
			\includegraphics[width=1.52in]{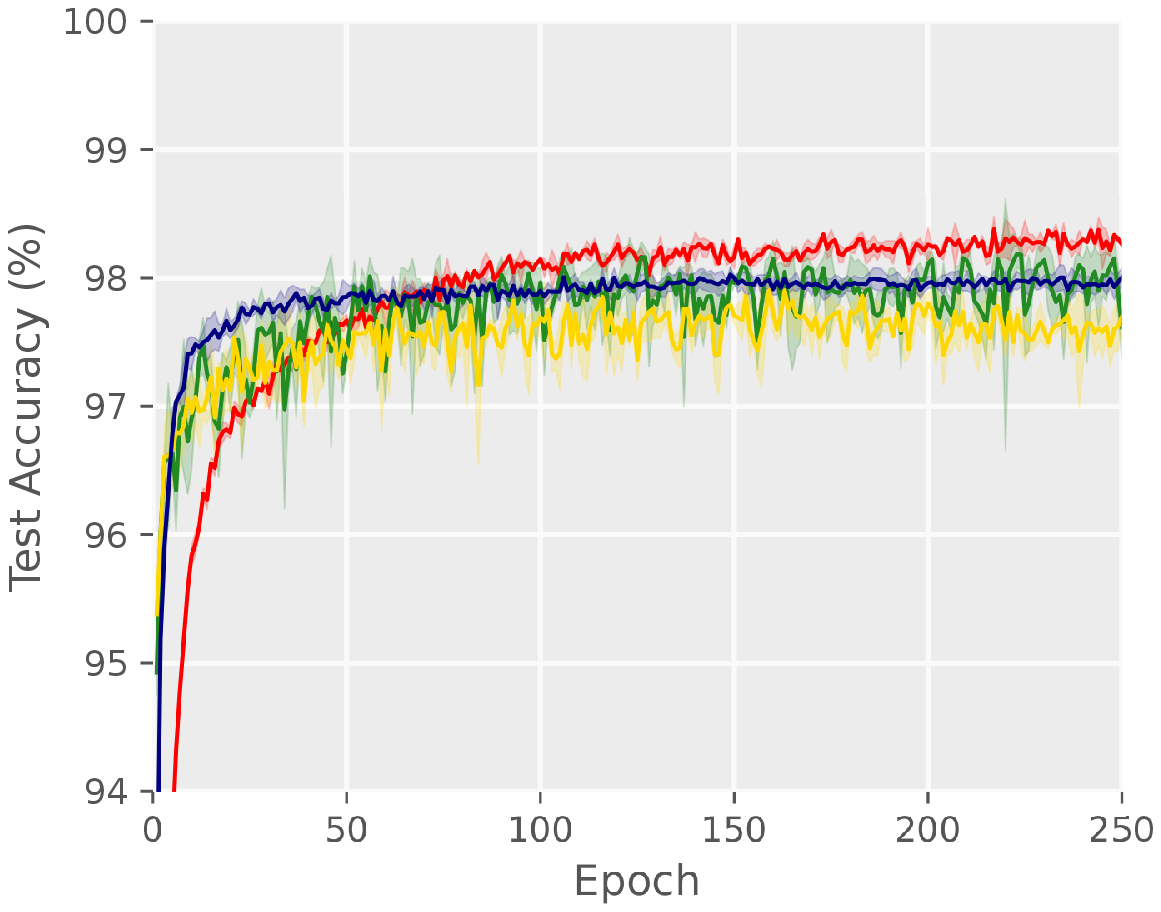}
			\centerline{\quad MNIST, $\alpha=0.8$}
	\end{minipage}}
	\subfigure{
		\begin{minipage}[b]{0.24\columnwidth}
			\centering
			\includegraphics[width=1.5in]{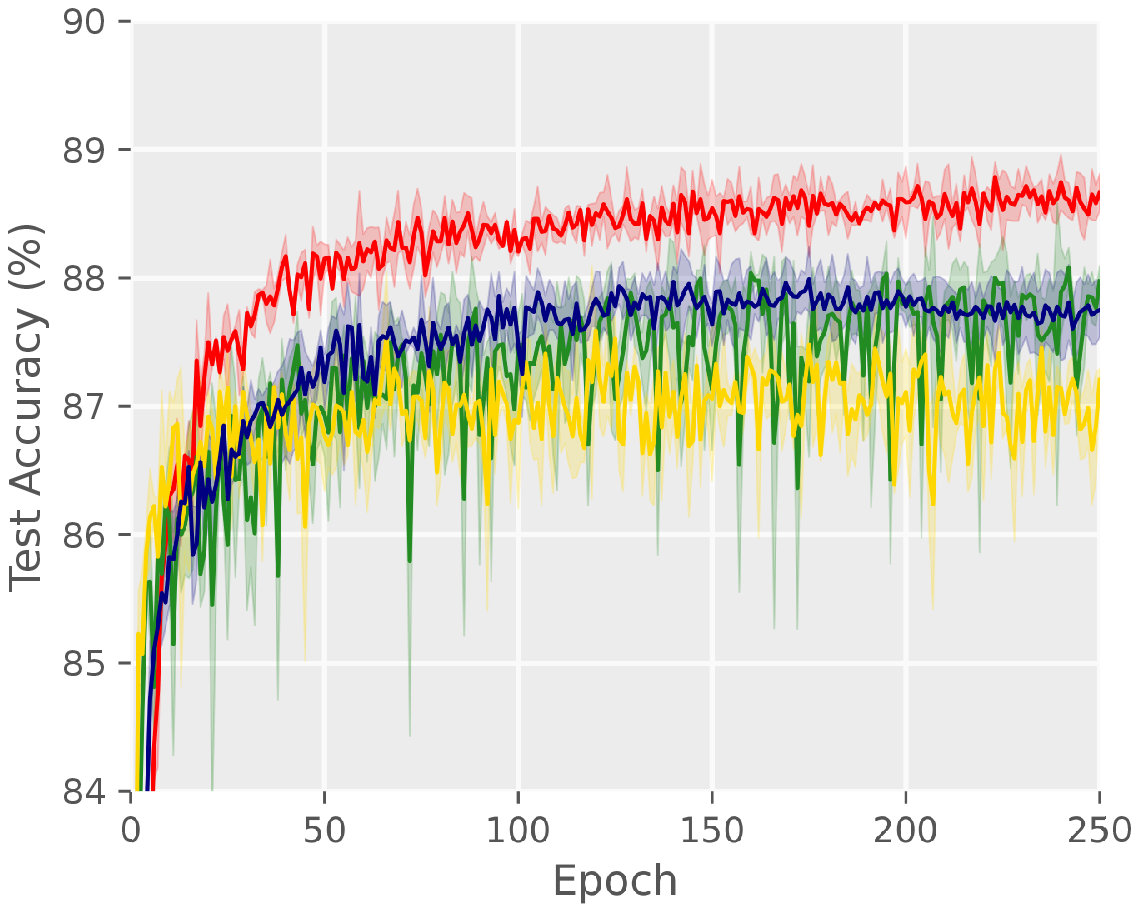}
			\centerline{\quad Fashion, $\alpha=0.8$}
	\end{minipage}}
	\subfigure{
		\begin{minipage}[b]{0.23\columnwidth}
			\centering
			\includegraphics[width=1.5in]{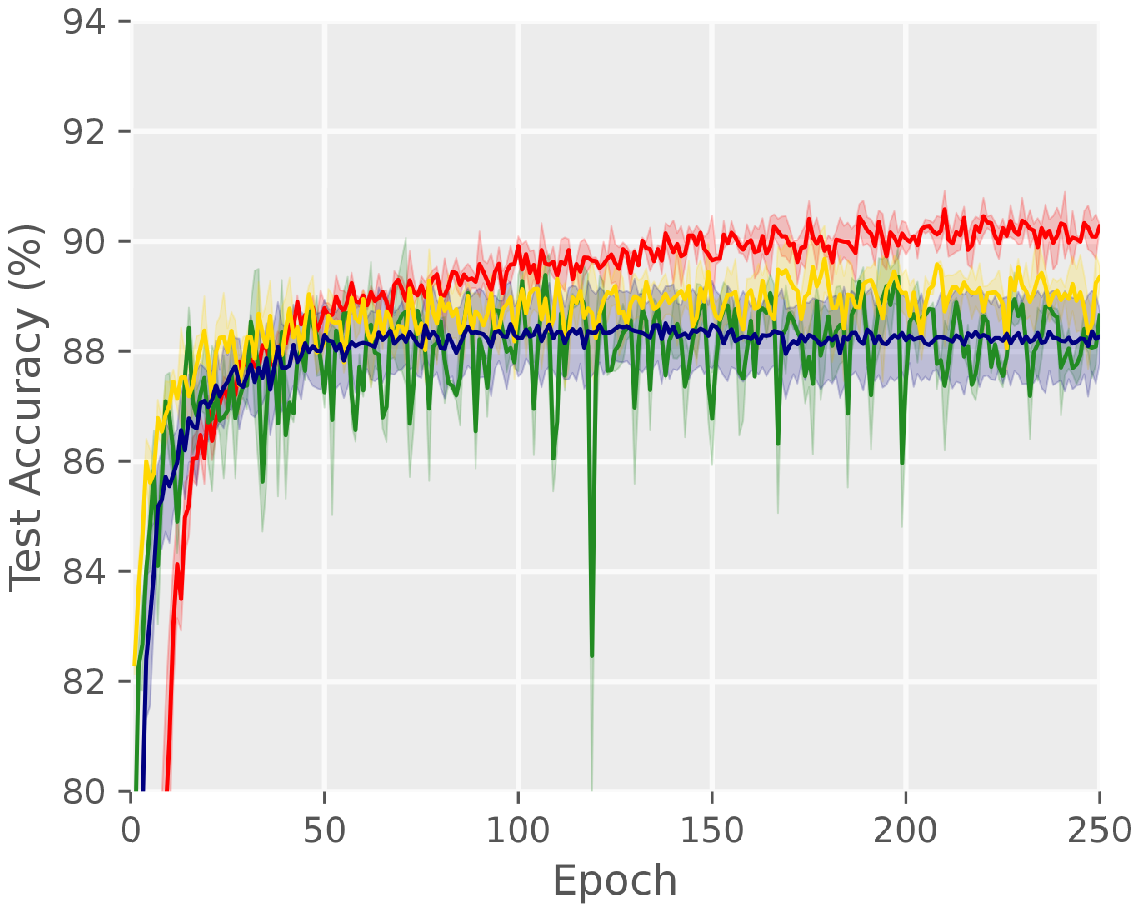}
			\centerline{\quad Kuzushiji, $\alpha=0.8$}
	\end{minipage}}
	\subfigure{
		\begin{minipage}[b]{0.23\columnwidth}
			\centering
			\includegraphics[width=1.5in]{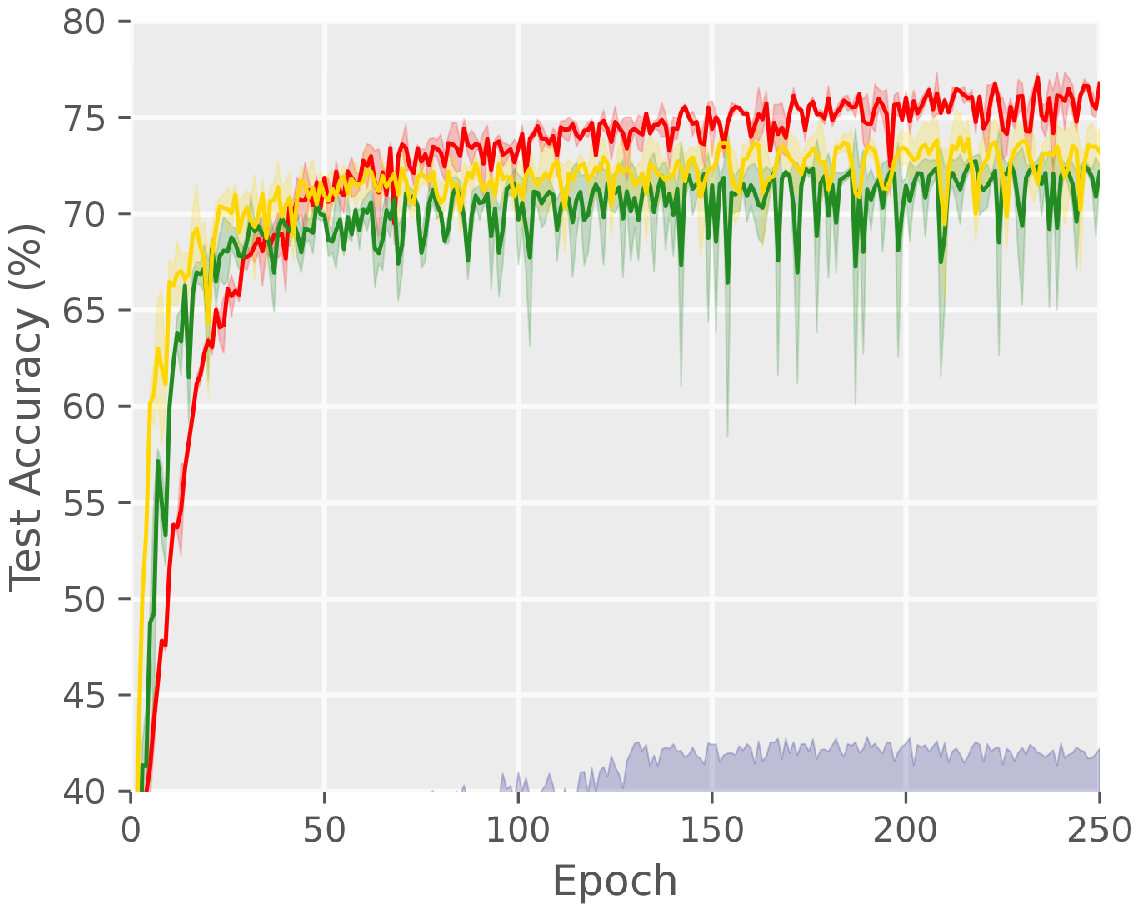}
			\centerline{\quad CIFAR-10, $\alpha=0.8$}
	\end{minipage}}
	
	\subfigure{
		\begin{minipage}[b]{0.24\columnwidth}
			\centering
			\includegraphics[width=1.52in]{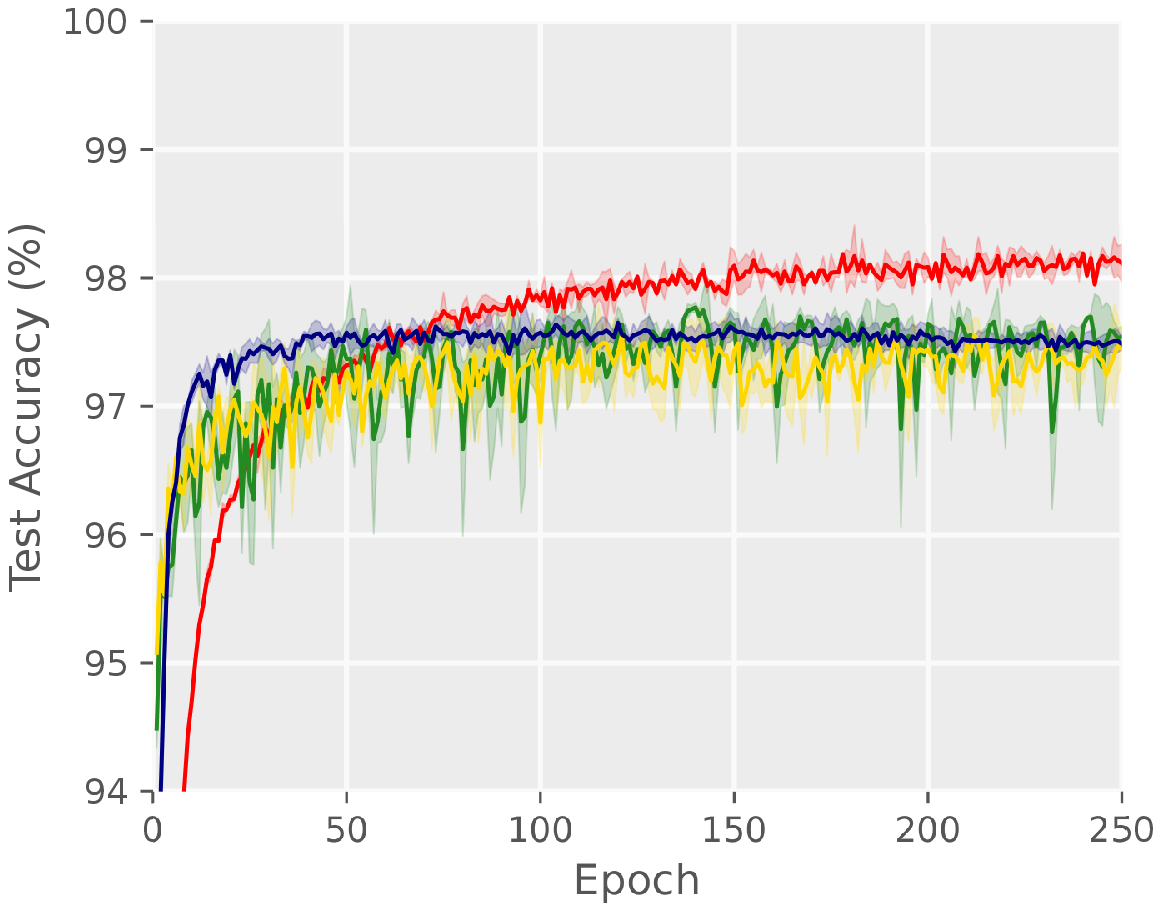}
			\centerline{\quad MNIST, $\alpha=0.7$}
	\end{minipage}}
	\subfigure{
		\begin{minipage}[b]{0.24\columnwidth}
			\centering
			\includegraphics[width=1.5in]{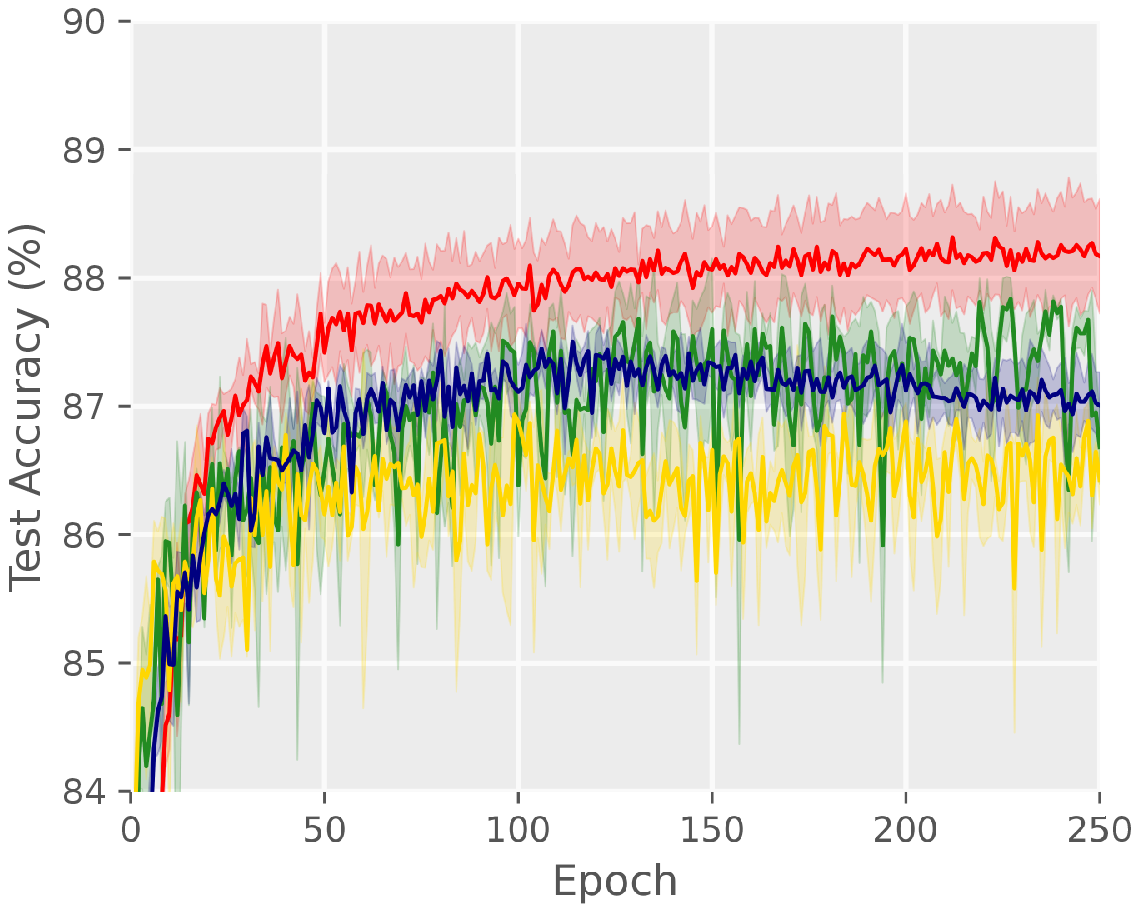}
			\centerline{\quad Fashion, $\alpha=0.7$}
	\end{minipage}}
	\subfigure{
		\begin{minipage}[b]{0.23\columnwidth}
			\centering
			\includegraphics[width=1.5in]{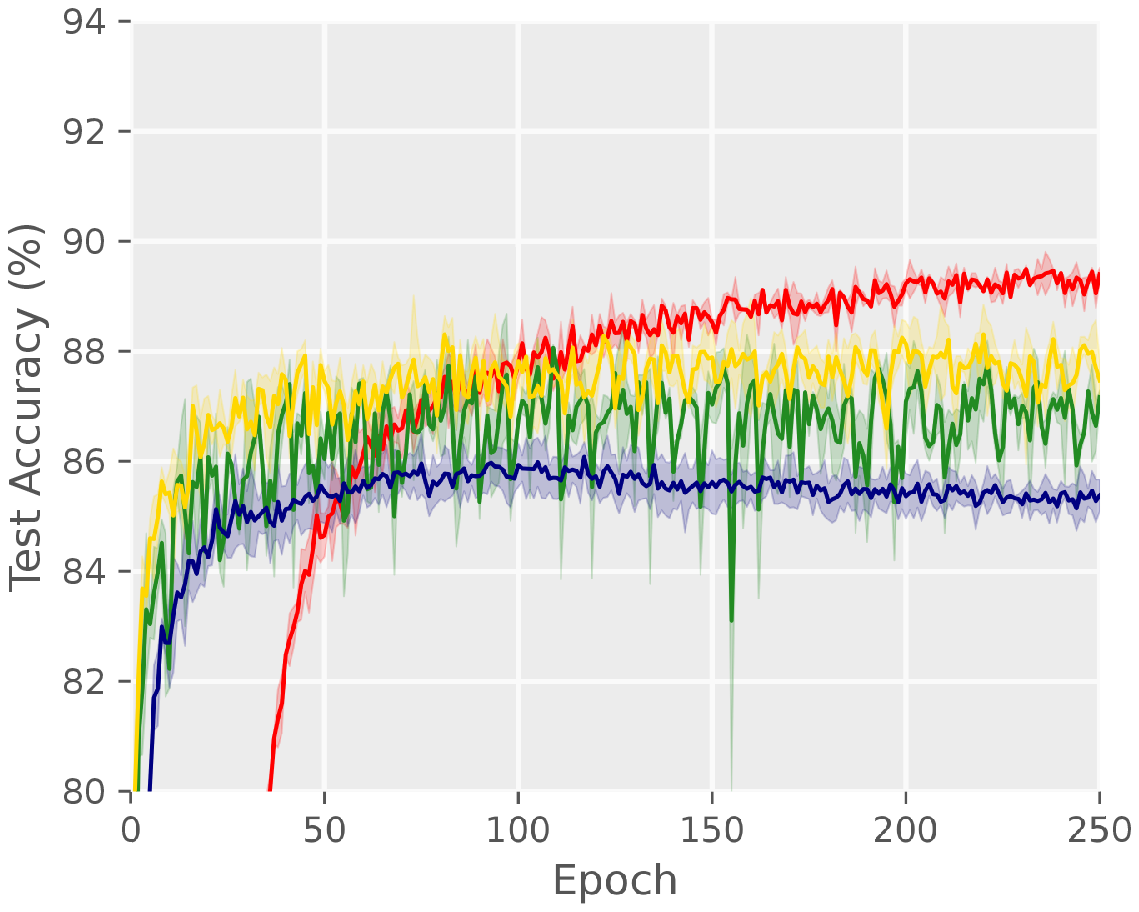}
			\centerline{\quad Kuzushiji, $\alpha=0.7$}
	\end{minipage}}
	\subfigure{
		\begin{minipage}[b]{0.23\columnwidth}
			\centering
			\includegraphics[width=1.5in]{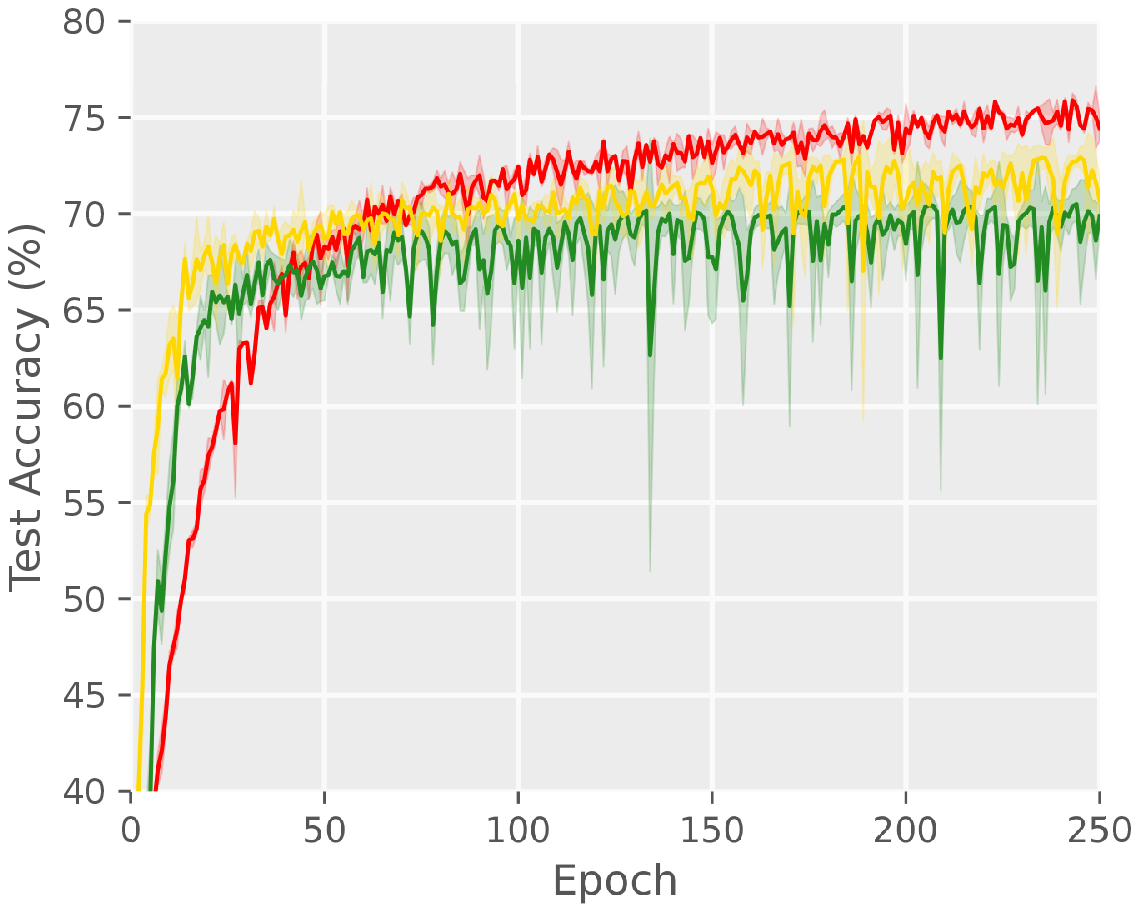}
			\centerline{\quad CIFAR-10, $\alpha=0.7$}
	\end{minipage}}
	
	\caption{Test accuracy with different $\alpha$. MLP was trained on MNIST, Fashion-MNIST, and Kuzushiji-MNIST, and DenseNet was trained on CIFAR-10. Dark colors represent the mean accuracy over 5 trials and light colors represent the standard error.}
	\label{fres2}
\end{figure*}

\paragraph{Baseline Methods.}
To evaluate the performance of the proposed PPL algorithm, we compared it with MCL, a special practice included in the general PPL framework, and two state-of-the-art consistent PL methods, the classifier-consistent (CC) method~\citep{feng2020provably} and the leveraged weighted (LW) method~\citep{wen2021leverage}.
The hyper-parameters were searched through five-fold cross-validation under the suggested settings in the original papers.
For MCL and LW, we used the log loss function and the cross-entropy loss function which achieved the overall best performance in the original papers.

\paragraph{Partial Label Generation Process.}

In the experiment of PCPL~\citep{feng2020provably}, the partial label $s$ for each instance was uniformly sampled from ${\mathcal{S}}$ given $y\in s$. The resulting distribution of $|s|$ is symmetric. The same generation process was also adopted in the experiment of MCL~\citep{feng2020learning} to generate the complementary label set $\bar{s}$. However, in most real-world cases, smaller complementary label sets are easier to collect. Motivated by this observation, we consider the $\alpha$-skewed distribution: $\bar{Q}_{k+1} = \alpha \bar{Q}_{k}, 0\leq k \leq K-1$. Here, $\bar{Q}_k := P(|\bar{s}|=k)$. Figure~\ref{skewandmcl} compares the PCPL distribution with the $\alpha$-skewed distribution.

Next, we define the average partial label size as $\mathbb{E}\left[|s|\right]$, which can measure the ambiguity and the hardness of PL problems. We summarize the average partial label size for different distributions in Table~\ref{averagesize}.

We can see that the ambiguity of PL problems increases as $\alpha$ decreases. Based on the $\alpha$-skewed distribution, we consider the following partial label generation process:
\begin{enumerate}
\item Instantiate $\bar{Q}_k$ according to $\alpha$;
\item For each instance $\boldsymbol{x}$, we first sample the size $|\bar{s}|$ from $\bar{Q}$. Then, we uniformly sample a complementary label set $\bar{s}$ with size $|\bar{s}|$. Finally, a partial label $s$ is generated by $s = {\mathcal{Y}}\setminus\bar{s}.$
\end{enumerate}

\begin{table}
	\caption{Empirical risk (mean$\pm$standard error) over 5 trials with different $\alpha$. The linear model was trained on MNIST, Fashion-MNIST, and Kuzushiji-MNIST, and ResNet was trained on CIFAR-10. Boldfaced numbers denote the method with the smallest standard error.}
	\vspace{0.2cm}
	\label{LossResult}
	\centering
	\resizebox{\columnwidth}{!}{
		\begin{tabular}{lccccc}
			\toprule
			Dataset & $\alpha$ & 
			MCL & PPL & 
			CC & LW
			\\ 
			\midrule
			\multirowcell{3}{MNIST}
			& 0.9 & 0.234$\pm$0.012 & 
			0.256$\pm$0.007  &
			\textbf{0.120$\pm$0.005}  & 
			0.193$\pm$0.008 \\ 
			
			& 0.8 & 0.231$\pm$0.012 & 
			\textbf{0.250$\pm$0.006}  &
			0.116$\pm$0.009 & 
			0.179$\pm$0.006  \\ 
			
			& 0.7  & 0.223$\pm$0.016 & 
			0.247$\pm$0.005  &
			\textbf{0.070$\pm$0.003} & 
			0.182$\pm$0.016  \\ 
			\hline
			\multirowcell{3}{Fashion-\\MNIST}      
			& 0.9 & 0.368$\pm$0.007 & 
			\textbf{0.409$\pm$0.003}  &
			0.235$\pm$0.008 & 
			0.248$\pm$0.011  \\ 
			& 0.8 & 0.354$\pm$0.009 & 
			\textbf{0.404$\pm$0.005}  &
			0.207$\pm$0.010 & 
			0.165$\pm$0.008  \\ 
			& 0.7 & 0.356$\pm$0.018 & 
			\textbf{0.412$\pm$0.005}  &
			0.223$\pm$0.006 & 
			0.195$\pm$0.007  \\ 
			\hline
			\multirowcell{3}{Kuzushiji-\\MNIST}      
			& 0.9 & 0.447$\pm$0.010 & 
			\textbf{0.539$\pm$0.003}  &
			0.270$\pm$0.005 & 
			0.308$\pm$0.010 \\ 
			& 0.8 & 0.432$\pm$0.009 & 
			0.530$\pm$0.005  &
			\textbf{0.228$\pm$0.004} & 
			0.297$\pm$0.012  \\ 
			& 0.7 & 0.431$\pm$0.011 & 
			\textbf{0.528$\pm$0.003}  &
			0.201$\pm$0.006 & 
			0.252$\pm$0.019    \\ 
			\hline
			\multirowcell{3}{CIFAR-10}      
			& 0.9 & 1.437$\pm$0.087 & 
			\textbf{0.750$\pm$0.040}  &
			0.928$\pm$0.048 & 
			0.589$\pm$0.081  \\ 
			& 0.8 & 1.981$\pm$0.160 & 
			\textbf{0.841$\pm$0.038}  &
			0.832$\pm$0.051 & 
			-  \\ 
			& 0.7 & 1.123$\pm$0.092 & 
			\textbf{0.893$\pm$0.022}  &
			0.882$\pm$0.037 & 
			-  \\
			\bottomrule\end{tabular}}
	
\end{table}

\begin{table}
	\caption{Empirical risk (mean$\pm$standard error) over 5 trials with different $\alpha$. MLP was trained on MNIST, Fashion-MNIST, and Kuzushiji-MNIST, and DenseNet was trained on CIFAR-10. Boldfaced numbers denote the method with the smallest standard error.}
	\vspace{0.2cm}
	\label{LossResult2}
	\centering
	\resizebox{\columnwidth}{!}{
		\begin{tabular}{lccccc}
			\toprule
			Dataset & $\alpha$ & 
			MCL & PPL & 
			CC & LW
			\\ 
			\midrule
			\multirowcell{3}{MNIST}
			& 0.9 & 0.073$\pm$0.005 & 
			0.069$\pm$0.010  &
			\textbf{0.043$\pm$0.004}  & 
			0.057$\pm$0.006 \\ 
			
			& 0.8 & 0.085$\pm$0.028 & 
			\textbf{0.109$\pm$0.004} &
			0.077$\pm$0.005 & 
			0.061$\pm$0.004  \\ 
			
			& 0.7  & 0.097$\pm$0.015 & 
			0.135$\pm$0.005  &
			\textbf{0.044$\pm$0.002} & 
			0.049$\pm$0.007  \\ 
			\hline
			
			\multirowcell{3}{Fashion-\\MNIST}      
			& 0.9 & 0.533$\pm$0.048 & 
			0.244$\pm$0.011  &
			\textbf{0.161$\pm$0.005} & 
			0.148$\pm$0.006  \\ 
			& 0.8 & 0.481$\pm$0.057 & 
			\textbf{0.232$\pm$0.004}  &
			0.134$\pm$0.004 & 
			0.150$\pm$0.004  \\ 
			& 0.7 & 0.549$\pm$0.059 & 
			\textbf{0.232$\pm$0.004}  &
			0.148$\pm$0.009 & 
			0.101$\pm$0.005  \\ 
			\hline
			\multirowcell{3}{Kuzushiji-\\MNIST}      
			& 0.9 & 0.111$\pm$0.013 & 
			0.128$\pm$0.017 &
			0.102$\pm$0.015  & 
			\textbf{0.090$\pm$0.011} \\ 
			& 0.8 & 0.130$\pm$0.011 & 
			\textbf{0.139$\pm$0.004} &
			0.086$\pm$0.009 & 
			0.073$\pm$0.011  \\ 
			& 0.7 & 0.171$\pm$0.052 & 
			\textbf{0.167$\pm$0.005}  &
			0.095$\pm$0.009 & 
			0.044$\pm$0.018    \\ 
			\hline
			\multirowcell{3}{CIFAR-10}      
			& 0.9 & 1.362$\pm$0.120 & 
			\textbf{0.612$\pm$0.030}  &
			0.553$\pm$0.047 & 
			0.555$\pm$0.045  \\ 
			& 0.8 & 1.330$\pm$0.232 & 
			\textbf{0.655$\pm$0.020}  &
			0.528$\pm$0.033 & 
			-  \\ 
			& 0.7 & 1.660$\pm$0.087 & 
			\textbf{0.708$\pm$0.015}  &
			0.431$\pm$0.039 & 
			-  \\
			\bottomrule\end{tabular}}
	
\end{table}

\paragraph{Experimental Results.}

We randomly sampled 10\% of the training set to construct a validation set. 
We selected the learning rate and weight decay from the set $\{10^{-6}, 10^{-5}, \ldots, 10^{-1}\}$ to achieve the best validation score. The mini-batch size was set to 256 and the epoch number was set to 250. 
The test performance was evaluated based on 5 trials on the four benchmark datasets.
We recorded the means and standard deviations of the test accuracy with different $\alpha$ were shown in Figure~\ref{fres1} and Figure~\ref{fres2}.

From the figures, we can see that PPL is always the best method or at least comparable to the best method with all the datasets and models.
We also notice that the test accuracy decreased as we decreased $\alpha$, which accords with our conjecture that the average partial label size can be used to measure the hardness of the PL problems. 
As the hardness of learning increases, the advantages of PPL become more apparent, while the performance of other baselines is greatly reduced, especially LW.
When the PL problem is too difficult (CIFAR-10, $\alpha=0.7$), LW learns little useful information, so not shown in the corresponding figures.
    
We further recorded the variance of the empirical risk of different methods in Table~\ref{LossResult} and Table~\ref{LossResult2}, and highlighted the results with the smallest standard deviation in boldface. 
As shown in the tables, the risk variance of PPL was the smallest in most cases. 
This goes some way to explaining the more robust risk estimator of PPL in comparison to other consistent PL methods, leading to stable performance.

\section*{Conclusions}
We presented a unified framework for \emph{learning from proper partial labels} which accommodates many previous problems settings as special cases and is substantially more general. We derived an unbiased estimator of the classification risk for proper partial label learning problems and theoretically established an estimation error bound for the proposed method. We also demonstrated the effectiveness of the proposed method through experiments on benchmark datasets.

\subsubsection*{Acknowledgments}
We thank Nontawat Charoenphakdee for helpful discussion. MS was supported by JST AIP Acceleration Research Grant Number JPMJCR20U3, Japan and the Institute for AI and Beyond, UTokyo.


\appendix
\section*{Appendix}
\everymath{\displaystyle}


\section{Proofs}\label{Appendix:Proofs}
\subsection{Proof of Proposition~\ref{propCondInd}}
If~\eqref{cmm1} and~\eqref{cmm2} hold, we can easily see that~\eqref{cmmAssumption} holds. Conversely, if~\eqref{cmmAssumption} holds, we have
 \begin{align*}
     p(s|y) &= \int_{{\mathcal{X}}}{ p(\boldsymbol{x}, s|y)\mathrm{d}\boldsymbol{x}} = \int_{{\mathcal{X}}}p(s|\boldsymbol{x}, y)p(\boldsymbol{x}|y)\mathrm{d}\boldsymbol{x}
     \\
     &= \int_{{\mathcal{X}}}C(s)\boldsymbol{1}\{y\in s\}p(\boldsymbol{x}|y)\mathrm{d}\boldsymbol{x} = C(s)\boldsymbol{1}\{y\in s\}\cdot\int_{{\mathcal{X}}}p(\boldsymbol{x}|y)\mathrm{d}\boldsymbol{x}
     \\
     &=p(s|\boldsymbol{x}, y) \int_{{\mathcal{X}}}p(\boldsymbol{x}|y)\mathrm{d}\boldsymbol{x}.
 \end{align*}
By noticing that $\int_{{\mathcal{X}}}p(\boldsymbol{x}|y)\mathrm{d}\boldsymbol{x} = 1$, we have
\begin{equation*}
    p(s|y) = p(s|\boldsymbol{x}, y) \int_{{\mathcal{X}}}p(\boldsymbol{x}|y)\mathrm{d}\boldsymbol{x} = p(s|\boldsymbol{x}, y).
\end{equation*}
Hence,~\eqref{cmm1} holds and we have
\begin{equation*}
    p(s|y)=p(s|\boldsymbol{x}, y)=C(s)\boldsymbol{1}\{y\in s\},
\end{equation*}
which concludes the proof.
\subsection{Proof of Proposition~\ref{proppcplstronger}}
First, CL is a special case of MCL and the CL assumption is different from the PCPL assumption when $K\geq 3$. Hence, we know the opposite is not true. Next, it suffices to prove that 
 \begin{equation*}
     Q_{k} = \frac{\binom{K-1}{k-1}}{2^{K-1} - 1}.
 \end{equation*}
 By definition, we have
 \begin{align*}
     Q_k &= P(|s| = k) = \sum_{|s| = k} p(s)=\sum_{|s|=k}\int_{{\mathcal{X}}}\sum_{i=1}^K p(\boldsymbol{x}, y=i, s)\mathrm{d}\boldsymbol{x}\\
     &=\sum_{|s|=k}\sum_{i=1}^K\int_{{\mathcal{X}}}p(s|\boldsymbol{x}, y=i)p(\boldsymbol{x}, y=i) \mathrm{d}\boldsymbol{x}\\
     &=\sum_{|s|=k}\sum_{i=1}^K\int_{{\mathcal{X}}}\frac{1}{2^{K-1}-1}\boldsymbol{1}\{i\in s\}p(\boldsymbol{x}, y=i) \mathrm{d}\boldsymbol{x}
     \\
     &=\frac{1}{2^{K-1}-1}\sum_{|s|=k}\sum_{i=1}^K\boldsymbol{1}\{i\in s\}p(y=i)\\
     &=\frac{1}{2^{K-1}-1}\sum_{i=1}^K\sum_{|s|=k}\boldsymbol{1}\{i\in s\}p(y=i).
 \end{align*}
 When $i$ is fixed in $s$, there are $\binom{K-1}{k-1}$ different choices of $s$ with size $k$. By noticing that $\sum_{i=1}^Kp(y=i) = 1$, we have
 \begin{align*}
     Q_k &= \frac{1}{2^{K-1}-1}\sum_{i=1}^K\binom{K-1}{k-1}p(y=i)=\frac{\binom{K-1}{k-1}}{2^{K-1} - 1},
 \end{align*}
 which completes the proof.

\subsection{Proof of Theorem~\ref{theorem:pplequivalence}}
     If~\eqref{propercond} holds, i.e., $p(s|\boldsymbol{x}, y) = C(\boldsymbol{x}, s)\boldsymbol{1}\{y\in s\}$, then we have
     \begin{align*}
         r_y(\boldsymbol{x}, s) &= \frac{p(y, s|\boldsymbol{x})}{p(s|\boldsymbol{x})} = \frac{p(y, s|\boldsymbol{x})}{\sum_{k\in s}p(y = k, s|\boldsymbol{x})}\\
         &= \frac{p(s|\boldsymbol{x}, y)p(y|\boldsymbol{x})}{\sum_{k\in s}p(s|\boldsymbol{x}, y = k)p(y=k|\boldsymbol{x})}\\
         &= \frac{C(\boldsymbol{x}, s)\boldsymbol{1}\{y \in s\}p(y|\boldsymbol{x})}{\sum_{k \in s}C(\boldsymbol{x}, s)\boldsymbol{1}\{k \in s\}p(y=k|\boldsymbol{x})}\\
         &= \frac{p(y|\boldsymbol{x})}{\sum_{k\in s}p(y=k|\boldsymbol{x})}\boldsymbol{1}\{y\in s\}.
     \end{align*}
     Conversely, if~\eqref{pplConfidence} holds, we have
     \begin{align*}
         p(s|\boldsymbol{x}, y) = \frac{p(y,s|\boldsymbol{x})}{p(y|\boldsymbol{x})} = \frac{p(y|\boldsymbol{x}, s)p(s|\boldsymbol{x})}{p(y|\boldsymbol{x})} = \frac{p(s|\boldsymbol{x})}{\sum_{k \in s}p(y=k|\boldsymbol{x})}\boldsymbol{1}\{y \in s\}.
     \end{align*}
     Hence, there exists a function \begin{equation*}
         C(\boldsymbol{x}, s) = \frac{p(s|\boldsymbol{x})}{\sum_{k \in s}p(y=k|\boldsymbol{x})}
     \end{equation*} 
     such that
     \begin{equation*}
         p(s|\boldsymbol{x}, y) = C(\boldsymbol{x}, s)\boldsymbol{1}\{y\in s\}.
     \end{equation*}

\subsection{Proof of Theorem~\ref{theorem:clc}}
By definition, we have
    \begin{align*}
        R(f;{\mathcal{L}}) &= \int_{{\mathcal{X}}}\sum_{j=1}^K{\mathcal{L}}(f(\boldsymbol{x}), j)p(\boldsymbol{x}, y=j)\mathrm{d}\boldsymbol{x}= \int_{{\mathcal{X}}}\sum_{j=1}^K\sum_{s\in{\mathcal{S}}}{\mathcal{L}}(f(\boldsymbol{x}), j)p(\boldsymbol{x}, y=j, s)\mathrm{d}\boldsymbol{x}\\
        &= \int_{{\mathcal{X}}}\sum_{j=1}^K\sum_{s\in{\mathcal{S}}}{\mathcal{L}}(f(\boldsymbol{x}), j)p(y=j|\boldsymbol{x}, s)p(\boldsymbol{x}, s)\mathrm{d}\boldsymbol{x}\\
        &= \int_{{\mathcal{X}}}\sum_{s\in{\mathcal{S}}}\sum_{j=1}^Kr_j(\boldsymbol{x}, s){\mathcal{L}}(f(\boldsymbol{x}), j)p(\boldsymbol{x}, s)\mathrm{d}\boldsymbol{x}\\
        &= \mathbb{E}_{p(\boldsymbol{x}, s)}\left[\sum_{j\in s}r_j(\boldsymbol{x}, s){\mathcal{L}}(f(\boldsymbol{x}), j)\right].
    \end{align*}

\subsection{Proof of Theorem~\ref{theorem:ppl}}
By Theorem~\ref{theorem:pplequivalence},~\eqref{pplConfidence} holds. By Theorem~\ref{theorem:clc}, we have
    \begin{align*}
        R(f;{\mathcal{L}}) &= \mathbb{E}_{p(\boldsymbol{x}, s)}\left[\sum_{j\in s}r_j(\boldsymbol{x}, s){\mathcal{L}}(f(\boldsymbol{x}), j)\right]\\& =\mathbb{E}_{p(\boldsymbol{x}, s)}\left[\sum_{j\in s}\frac{p(y=j|\boldsymbol{x})}{\sum_{k\in s}p(y=k|\boldsymbol{x})}{\mathcal{L}}(f(\boldsymbol{x}), j)\right], 
    \end{align*}
    which completes the proof.
    
\subsection{Proof of Theorem~\ref{theorem:bound}}
We first introduce the following lemma:
    \begin{lemma}
    \label{lemma1}
    The following inequality holds:
    \begin{equation*}
        0\leq R({\hat{f}_\mathrm{PPL}}) - R(f^*) \leq 2\sup_{f\in{\mathcal{F}}}\left\lvert R(f)-{\hat{R}_\mathrm{PPL}}(f)\right\rvert.
    \end{equation*}
    \end{lemma}
    \begin{proof} By definition, $R({\hat{f}_\mathrm{PPL}})-R(f^*)\geq 0$. Also, we have
    \begin{align*}
        R({\hat{f}_\mathrm{PPL}}) - R(f^*) &= \left(R({\hat{f}_\mathrm{PPL}}) - {\hat{R}_\mathrm{PPL}}({\hat{f}_\mathrm{PPL}})\right)\\ &+ \left({\hat{R}_\mathrm{PPL}}({\hat{f}_\mathrm{PPL}}) - {\hat{R}_\mathrm{PPL}}(f^*)\right) \\&+ \left({\hat{R}_\mathrm{PPL}}(f^*) - R(f^*)\right) \\
        &\leq \left(R({\hat{f}_\mathrm{PPL}}) - {\hat{R}_\mathrm{PPL}}({\hat{f}_\mathrm{PPL}})\right) + \left({\hat{R}_\mathrm{PPL}}(f^*) - R(f^*)\right) \\
        &\leq 2\sup_{f\in{\mathcal{F}}}\left\lvert R(f)-{\hat{R}_\mathrm{PPL}}(f)\right\rvert.
    \end{align*}
    \end{proof}
    The Rademacher complexity~\cite{mohri2018foundations} is defined as follows:
    \begin{definition}[Empirical Rademacher Complexity]
    Let ${\mathcal{G}}$ be a class of functions mapping ${\mathcal{Z}}$ to $\mathbb{R}$ and $S=(z_1, \dots, z_n) \in {\mathcal{Z}}^n$ a fixed sample of size $n$. Then, the empirical Rademacher complexity of ${\mathcal{G}}$ with respect to the sample S is defined as
    \begin{equation*}
        \hat{{\mathcal{R}}}_S({\mathcal{G}}) = \mathbb{E}_{\boldsymbol{\sigma}}\left[\sup_{g\in {\mathcal{G}}}\frac{1}{n}\sum_{i=1}^n\sigma_ig(z_i)\right],
    \end{equation*}
    where $\boldsymbol{\sigma}=(\sigma_1, \ldots, \sigma_n)$, with $\sigma_i$s independent uniform random variables taking values in $\{-1,+1\}$.

    \end{definition}
    \begin{definition}[Rademacher Complexity]
    Suppose the sample $S$ of size $n$ is drawn independently from the distribution denoted by a probability density function $p$. The Rademacher complexity of ${\mathcal{G}}$ with respect to $p$ is defined as
    \begin{equation*}
        {\mathcal{R}}_n({\mathcal{G}})=\mathbb{E}_{z_i\sim p}\left[\hat{{\mathcal{R}}}_S({\mathcal{G}})\right].
    \end{equation*}
    \end{definition}
    We introduce a class of functions defined on ${\mathcal{X}}\times{\mathcal{S}}$ according to Eq.~\eqref{eq:pplrisk}:
    \begin{equation*}
        {\mathcal{G}}=\{(\boldsymbol{x}, s)\rightarrow \sum_{j\in s}\frac{p(y=j|\boldsymbol{x})}{\sum_{k\in s}p(y=k|\boldsymbol{x})}{\mathcal{L}}(f(\boldsymbol{x}), j):f\in{\mathcal{F}}\}.
    \end{equation*}
    Then, the Rademacher complexity of ${\mathcal{G}}$ with respect to $p(\boldsymbol{x}, s)$ is given as:
    \begin{equation*}
        {\mathcal{R}}_n({\mathcal{G}})=\mathbb{E}_{(\boldsymbol{x}_i, s_i)\sim p}\left[\mathbb{E}_{\boldsymbol{\sigma}}\left[\sup_{g\in{\mathcal{G}}}\frac{1}{n}\sum_{i=1}^n\sigma_ig(\boldsymbol{x}_i, s_i)\right]\right].
    \end{equation*}
    We have the following lemma:
    \begin{lemma}\label{lemma2}
    Suppose $M:=\sup_{\boldsymbol{x}\in{\mathcal{X}},y\in{\mathcal{Y}},f\in{\mathcal{F}}}{\mathcal{L}}(f(\boldsymbol{x}),y)<\infty$, then, for any $\delta>0$, the following holds with probability at least $1-\delta$:
    \begin{equation*}
        \sup_{f\in{\mathcal{F}}}\left\lvert R(f)-{\hat{R}_\mathrm{PPL}}(f)\right\rvert \leq 2{\mathcal{R}}_n({\mathcal{G}})+M\sqrt{\frac{\log\frac{2}{\delta}}{2n}}.
    \end{equation*}
    \end{lemma}
    \begin{proof}
     For a sample $S$, we define $\phi(S) = \sup_{f\in{\mathcal{F}}}\left\{R(f)-{\hat{R}_\mathrm{PPL}}(f)\right\}$. Suppose we replace an example $(\boldsymbol{x}_i, s_i)$ in the sample $S$ with another example $(\boldsymbol{x}'_i, s'_i)$, the change of $\phi(S)$ is no greater than $$\sup_{g\in{\mathcal{G}}}\frac{g(\boldsymbol{x}_i, s_i) - g(\boldsymbol{x}'_i, s'_i)}{n}\leq \frac{M}{n},$$ since ${\mathcal{L}}$ is bounded by $M$. Then, by \emph{McDiarmid's inequality}~\cite{mcdiarmid1989method}, for any $\delta > 0$, with probability at least $1 - \delta/2$, the
following holds:
\begin{equation*}
    \phi(S)\leq\mathbb{E}_{(\boldsymbol{x}_i, s_i)\sim p}\left[\phi(S)\right] + M\sqrt{\frac{\log\frac{2}{\delta}}{2n}}.
\end{equation*}
It is a routine work~\cite{mohri2018foundations} to show
$\mathbb{E}_{(\boldsymbol{x}_i, s_i)\sim p}\left[\phi(S)\right]\leq 2{\mathcal{R}}_n({\mathcal{G}})$. Hence, the following holds with probability at least $1 - \delta/2$:
\begin{equation}\label{lemma2eq1}
    \sup_{f\in{\mathcal{F}}}\left\{R(f)-{\hat{R}_\mathrm{PPL}}(f)\right\}\leq 2{\mathcal{R}}_n({\mathcal{G}}) + M\sqrt{\frac{\log\frac{2}{\delta}}{2n}}.
\end{equation}
Similarly, we can prove that the following holds with probability at least $1 - \delta/2$: 
\begin{equation}\label{lemma2eq2}
    \sup_{f\in{\mathcal{F}}}\left\{{\hat{R}_\mathrm{PPL}}(f)-R(f)\right\}\leq 2{\mathcal{R}}_n({\mathcal{G}}) + M\sqrt{\frac{\log\frac{2}{\delta}}{2n}}.
\end{equation}
We complete the proof by combining~\eqref{lemma2eq1} and~\eqref{lemma2eq2}.
    \end{proof}
    Next, we bound the Rademacher complexity ${\mathcal{R}}_n({\mathcal{G}})$ by the following lemma~\cite{feng2020provably}:
    \begin{lemma}\label{lemma3}
    Suppose that the loss ${\mathcal{L}}(f(\boldsymbol{x}), s)$ is $\rho$-Lipschitz with respect to $f(\boldsymbol{x})$ for all $y\in{\mathcal{Y}}$. Then, the following inequality holds:
    \begin{equation*}
        {\mathcal{R}}_n({\mathcal{G}})\leq\sqrt{2}\rho\sum_{i=1}^K{\mathcal{R}}_n({\mathcal{F}}_i).
    \end{equation*}
    \end{lemma}
    \begin{proof}
     Let $\Pi=\{(\boldsymbol{x}, y)\rightarrow{\mathcal{L}}(f(\boldsymbol{x}), y): f\in {\mathcal{F}}\}$. Notice that the candidate label confidence $r_y(\boldsymbol{x}, s)$ satisfies that $0\leq r_y(\boldsymbol{x}, s)\leq 1$ and that $\sum_{k\in s}r_k(\boldsymbol{x}, s) = 1$. In this way, we can obtain ${\mathcal{R}}_n({\mathcal{G}})\leq {\mathcal{R}}_n(\Pi)$. Since ${\mathcal{L}}$ is $\rho$-Lipschitz with respect to $f(\boldsymbol{x})$, following the Rademacher vector contraction inequality~\cite{maurer2016vector}, we have ${\mathcal{R}}_n(\Pi)\leq\sqrt{2}\rho\sum_{i=1}^K{\mathcal{R}}_n({\mathcal{F}}_i)$, which concludes the proof.
    \end{proof}
    Finally, the proof of Theorem~\ref{theorem:ppl} is completed by combining Lemma~\ref{lemma1}, Lemma~\ref{lemma2} and Lemma~\ref{lemma3}.
    
\section{Experiment Details}\label{Appendix_Expdetails}

    We list here the details of the four benchmark datasets used in our experiments:
    \paragraph{MNIST:} 
    A 10-class dataset containing handwritten digits from 0 to 9. MNIST has in total 60,000 training images and 10,000 test images. Each instance is a 28 $\times$ 28 grayscale image.
    \paragraph{Fashion-MNIST:}
    A 10-class dataset of fashion items. Fashion-MNIST has in total 60,000 training images and 10,000 test images. Each instance is a 28 $\times$ 28 grayscale image.
    \paragraph{Kuzushiji-MNIST:}
    A 10-class dataset of cursive Japanese (Kuzushiji) characters. Kuzushiji-MNIST has in total 60,000 training images and 10,000 test images. Each instance is a 28 $\times$ 28 grayscale image.
    \paragraph{CIFAR-10:}
    A 10-class dataset of colored images (airplane, automobile, bird,
cat, deer, dog, frog, horse, ship, and truck). CIFAR-10 has in total 50,000 training images and 10,000 test images. Each instance is a 32 $\times$ 32 $\times$ 3 colored image.

    The information of each dataset and its corresponding models is summarized in Table \ref{dataset}.

    \begin{table}[t]\centering
    \caption{Information of datasets and corresponding models.}
    \vspace{0.2cm}
        \newcommand{\tabincell}[2]{\begin{tabular}{@{}#1@{}}#2\end{tabular}}
        \begin{tabular}{ c|c c c|c } 
        \hline
        Dataset & \tabincell{c}{\#Train} & \tabincell{c}{\#Test} & \tabincell{c}{Dimension} & 
        \tabincell{c}{Model} \\ 
        \hline
        MNIST & 60,000 & 10,000 & 784  & Linear, MLP \\ 

        Fashion-MNIST & 60,000 & 10,000 & 784  & Linear, MLP \\ 
        Kuzushiji-MNIST & 60,000 & 10,000 & 784 &Linear, MLP \\ 
        CIFAR-10 & 50,000 & 10,000 & 3,072 &  ResNet, DenseNet \\ 
        \hline
        \end{tabular}
        
        \label{dataset}
    \end{table}

\end{document}